%% file: neurips_2024.tex
\pdfoutput=1
\documentclass{article}

\usepackage{smile}

\input{macros}

\PassOptionsToPackage{numbers, compress}{natbib}

    \usepackage[final]{neurips_2024}

\usepackage[utf8]{inputenc} 
\usepackage[T1]{fontenc}    
\usepackage{hyperref}       
\usepackage{url}            
\usepackage{booktabs}       
\usepackage{amsfonts}       
\usepackage{nicefrac}       
\usepackage{microtype}      
\usepackage{xcolor}         
\usepackage{amssymb}

\title{On the Optimality of Dilated Entropy and Lower Bounds for Online Learning in Extensive-Form Games}

\usepackage{float}
\usepackage[outline]{contour}
\usepackage{tikz}
\input{gabri_figure}

\author{%
  Zhiyuan Fan  \\
  MIT\\
  \texttt{fanzy@mit.edu} \\
  \And
  Christian Kroer \\
  Columbia University \\
  \texttt{christian.kroer@columbia.edu} \\
  \AND
  Gabriele Farina \\
  MIT \\
  \texttt{gfarina@mit.edu} \\
}

\begin{document}

\maketitle

%
%
%
%
%
%
\begin{abstract}
First-order methods (FOMs) are arguably the most scalable algorithms for equilibrium computation in large extensive-form games. To operationalize these methods, a distance-generating function, acting as a regularizer for the strategy space, must be chosen. 
The ratio between the strong convexity modulus and the diameter of the regularizer is a key parameter in the analysis of FOMs.
A natural question is then: what is the \emph{optimal} distance-generating function for extensive-form decision spaces? In this paper, we make a number of contributions, ultimately establishing that the weight-one dilated entropy (DilEnt) distance-generating function is optimal up to logarithmic factors. 
The DilEnt regularizer is notable due to its iterate-equivalence with Kernelized OMWU (KOMWU)---the algorithm with state-of-the-art dependence on the game tree size in extensive-form games---when used in conjunction with the online mirror descent (OMD) algorithm. However, the standard analysis for OMD is unable to establish such a result; the only current analysis is by appealing to the iterate equivalence to KOMWU. 
We close this gap by introducing a pair of primal-dual \emph{treeplex} norms, which we contend form the natural analytic viewpoint for studying the strong convexity of DilEnt. 
Using these norm pairs, we recover the diameter-to-strong-convexity ratio that predicts the same performance as KOMWU. Along with a new regret lower bound for online learning in sequence-form strategy spaces, we show that this ratio is nearly optimal.
Finally, we showcase our analytic techniques by refining the analysis of Clairvoyant OMD when paired with DilEnt, establishing an $\mathcal{O}(n \log |\mathcal{V}| \log T/T)$ approximation rate to coarse correlated equilibrium in $n$-player games, where $|\mathcal{V}|$ is the number of reduced normal-form strategies of the players, establishing the new state of the art.

%
%
\end{abstract}

\section{Introduction}

Extensive-form games (EFG) are a popular framework for modeling sequential games with imperfect information.
The framework has been widely used to build superhuman AI agents in real-world imperfect information games 
\citep{bowling2015heads, moravvcik2017deepstack, brown2018superhuman, brown2019superhuman}.
Several notions of equilibrium, including Nash equilibrium \citep{nash1950equilibrium} in two-player zero-sum and coarse correlated equilibrium in general multiplayer EFGs, can be computed in polynomial time in the size of the game tree under the standard hypothesis of perfect recall \citep{von1996efficient, papadimitriou2008computing, huang2008computing, jiang2011polynomial}.
These polynomial-time algorithms, however, require running the ellipsoid method or polynomial algorithm for linear programming, both of which are impractical for large-scale games, due to the high memory usage and large per-iteration computational costs~\citep{sandholm2010state}.

Instead, fast iterative methods based on convex first-order optimization methods (FOMs) \citep{zinkevich2007regret, chiang2012online, tammelin2015solving, brown2019solving, farina2019online, kroer2020faster, farina2021better, lee2021last, liu2022power, liu2024policy} are commonly used to find an approximate equilibrium. These iterative methods define strategy update rules that each player can apply iteratively while training in self-play with other players, and that guarantee ergodic convergence to the set of equilibria in the long run. Three popular classes of such FOMs are employed in EFGs: methods based on online mirror descent (OMD) \citep{beck2003mirror, chiang2012online}, methods based on the counterfactual regret minimization framework \citep{zinkevich2007regret}, and (in the context of two-player zero-sum games specifically) accelerated offline methods such as mirror prox \citep{nemirovski2004prox} and the excessive gap technique \citep{nesterov2005excessive} algorithm.
In general, these methods are all proximal methods---that is, they perform a generalized notion of projected gradient descent step at each iteration. Some do this explicitly, including OMD and mirror prox, while others do it implicitly, including the counterfactual regret minimization algorithm, which runs proximal steps locally at each decision point \citep{farina2021faster}.

In all methods mentioned above, except CFR,\footnote{In CFR, the gradient steps are projected onto the nonnegative cone locally at each decision point of the game, and then renormalized to be a valid probability distribution over the actions.} the constraint set for the proximal step (\emph{i.e.}, the set on which gradient steps must be projected onto) is the strategy polytope of the EFG. The proximal steps are parameterized by a choice of \emph{distance-generating function (DGF)} for the strategy polytope, which acts as a regularizer. 
The performance of FOMs is sensitive to the properties of the DGF. In particular, two qualities are often desired:
(1) the ratio between the diameter of the feasible domain (as measured with the DGF) and the strong convexity modulus, with respect to a given norm, of the DGF must be as small as possible; and (2) projections with respect to the DGF onto the feasible set should take linear time in the dimension of the set. 

In EFGs, the only DGF family that satisfies the second requirement is based on the framework of \emph{dilated} regularization introduced by \citet{hoda2010smoothing}.
Within this framework, \citet{kroer2020faster} gave the first explicit strong convexity bounds based on the dilation framework, specifically for the \emph{dilated entropy DGF}. By combining optimistic regret minimizers for general convex sets with this DGF, one gets an algorithm that achieves a $T^{-1}$
convergence rate for two-player zero-sum EFGs. Subsequent work by \citet{farina2021better} introduced the \emph{dilated global entropy} DGF with an improved diameter-to-strong-convexity ratio.
By plugging their DGF results into the genereic OMD regret bound, one immediately achieves a regret bound of $\cO(\|\cQ\|_1 \sqrt{\log |\cA|} \sqrt{T})$. 
This was the state-of-the-art regret bound when introduced, in terms of dependence on game constants. Moreover, until now, it was the best bound known to be achievable through the direct application of OMD regret bounds combined with DGF results.
However, \citet{farina2022kernelized} developed a, seemingly, different approach based on \emph{kernelization}, which is a way to simulate, in linear time in the EFG size, the results of applying optimistic multiplicative weights (OMWU) on the normal-form reduction of an EFG. They call their algorithm \emph{KOMWU}. KOMWU achieves a better, and now state-of-the-art, regret bound $\cO(\sqrt{\log |\cV|} \sqrt{T})$ in online learning with full-information feedback.
Based on their result, two open questions emerged:
1) is this the best possible regret bound that one can achieve?
2) is it possible to achieve such a bound directly using the standard OMD machinery, without resorting to this kernelization trick?
\citet{bai2022efficient} made highly interesting progress on the second question: they show that, in fact, the KOMWU algorithm is iterate equivalent to OMD, with the specific version of dilated entropy that uses weight one everywhere. However, their result only shows a state-of-the-art rate by equivalence to KOMWU, and it is still unknown whether this state-of-the-art rate is achieveable directly through results on DGF properties and the standard OMD regret bound.
In this paper, we answer these two open questions, by answering the following question:
\begin{center}
    \emph{What is the optimal DGF for FOMs in solving EFGs?}
\end{center}

\renewcommand{\arraystretch}{1.2}
\begin{table}[t]
\centering
\setlength{\tabcolsep}{3.2mm}
\begin{tabular}{lccc}
\toprule
\rowcolor{white}
Regularizer & Norm pair & $|\cD|/\mu$ ratio & Max gradient norm \\
\midrule
\rowcolor{white}
Dilated Entropy \citep{kroer2020faster} & $\ell_1$ and $\ell_\infty$ norms & $\cO(2^{D} \|\cQ\|_1^2 \log |\cA|)$ & $\le 1$ \\
Dilated Gl. Entropy \citep{farina2021better} & $\ell_1$ and $\ell_\infty$ norms & $\cO(\|\cQ\|_1^2 \log |\cA|)$ & $\le 1$ \\
\rowcolor{LightGray}
DilEnt (\textbf{this paper}) & treeplex norms & $\ln |\cV|$ &  $\le 1$ \\
\bottomrule
\end{tabular}   
\caption{Comparison of the diameter-to-strong-convexity ($|\cD|/\mu$) ratio with prior results in DGFs for EFGs, where the ``Norm pair'' indicates the primal norm used in establishing the strong convexity, and its dual. ``Max gradient norm'' indicates the maximum norm---measured in the dual of the norm with respect to which each DGF is strongly convex---of any reward vector, or the gradient of utility function, that can be encountered during optimization, assuming that all payoffs at the terminal nodes of the EFG are in the range $[0,1]$.
$D$ denotes the depth of the tree, $\Qsize$ the tree size (see \cref{sec:preliminaries}), $|\cA|$ the maximum number of actions, and $|\cV|$ the number of reduced normal-form strategies. We remark that $\ln |\cV| \leq \cO(\|\cQ\|_1 \log |\cA|)$.}
\label{tab:rw-d}
\end{table}

We show that weight-one dilated entropy (DilEnt) is indeed the optimal DGF for FOMs in solving EFGs with full-information feedback, in terms of the diameter-to-strong-convexity ratio ($|\cD|/\mu$), up to logarithmic factors.
We note that the diameter-to-strong-convexity ratio of the regularizer is a key factor in the performance of FOMs. Intuitively, performance degrades as the diameter increases (since there is more ``space'' to search), and improves as the regularizer becomes more bowl-shaped (i.e., strongly convex). Consequently, a smaller diameter-to-strong-convexity ratio leads to better performance of the corresponding FOM.
Our contributions can be summarized as follows:
\begin{itemize}[leftmargin=*]
    \item We introduce a pair of primal-dual treeplex norms for the extensive-form decision space. These norms establish an improved framework for analyzing FOMs in EFGs, leading to results with better dependence on the size of the game. Based on this framework, we derive a new state-of-the-art diameter-to-strong-convexity ratio among all known DGFs for the EFG strategy spaces (see \cref{tab:rw-d} for comparison).
    By combining this new results with the standard OMD regret bound, we establish a regret upper bound that aligns with the results achieved by KOMWU.

    \item By establishing a matching regret lower bound, we identify a minimum diameter-to-strong-convexity ratio for any regularizer. We find that the DilEnt regularizer achieves the optimal ratio up to logarithmic factors, making it a natural candidate for FOMs on EFGs. 

    \item 
    An advantage of our new results, as compared to showing a regret bound through the KOMWU equivalence, is that our DGF result can also be combined with other algorithmic setups.
    As an example of our results, we show that by equipping Clairvoyant OMD~\citep{farina2022clairvoyant} with the DilEnt DGF, we enable convergence to a coarse correlated equilibrium at a rate of $\cO(n \log |\mathcal{V}| \log T/T)$ in $n$-player EFGs. This improves upon the previous results of $\cO(n \log |\mathcal{V}| \log^4 T/T)$ by \citet{farina2022kernelized} and $\cO(n \cdot |\cJ\times\cA| \Qsize^2 \log T/T)$ by \citet{farina2022near}, establishing a new state-of-the-art rate.
\end{itemize}

\section{Related Works}
\paragraph{Equilibrium Computation in General-Sum Extensive-Form Games}
The first line work in the equilibrium computation on general-sum EFGs used linear program (LP) which can be solved efficiently \citep{papadimitriou2008computing, jiang2011polynomial}. However, due to the large exponent of LP solvers, it is impractical to run such algorithms on large-scale games. The modern equilibrium computation using fast-iterative methods: \citet{syrgkanis2015fast} introduced the RVU property on the regret bound for a broad class of optimistic no-regret learning algorithms. With that property, they
demonstrated that the individual regret of each player grows as $T^{1/4}$ in general games, thus leading to a $T^{-3/4}$ converge rate to the coarse correlated equilibrium (CCE). A near-optimal bound of order $\log^4(T)$ was established by \citet{daskalakis2021near}, which implies a fast convergence rate of order $\tilde \cO(1/T)$. Subsequent work by \citet{farina2022kernelized} generalized the result to a class of polyhedral games that includes EFG.
Concurrently, \citet{piliouras2022beyond} introduced the Clairvoyant MWU. Although the algorithm is not non-regret learning, a subset of the steps converge to CCE with a rate of $\log T / T$. \citet{farina2022clairvoyant} showed that the algorithm is an instantiation of the conceptual proximal methods, which has been studied in the literature of FOMs \citep{chen1993convergence, nemirovski2004prox}. Using another technique, \citet{farina2022near} also achieved this rate with worse game size dependence. 
Another class of fast iterative methods follows from counterfactual regret minimization (CFR) \citep{zinkevich2007regret}, which guarantees a regret bound of order $\sqrt{T}$. \citet{farina2019stable} showed that running OMWU at each decision point achieves a $T^{1/4}$ external regret, thus leading to a $T^{-3/4}$ approximation rate to the CCE. Although CFR has a weaker guarantee on convergence rate, variants of the algorithm are widely used in practice due to their superior practical performance \citep{kroer2018solving}.

\paragraph{Regret Lower Bounds in Extensive-Form Games}

Several works have studied lower bounds in EFGs across various settings. \citet{koolen2010hedging} established a lower bound dependent on the number of orthogonal strategy profiles in the decision set for structured games, including EFGs, resulting in a bound of $\Omega(\sqrt{T \log |\cA|})$. 
\citet{syrgkanis2015fast} demonstrated that in two-player zero-sum games, if one player uses MWU while the other best responds, the former must endure a regret of at least $\Omega(\sqrt{T})$. Similarly, \citet{chen2020hedging} gave the same lower bound when both players use MWU.
For equilibrium computation, \citet{anagnostides2023complexity} analyzed the sparse-CCE in EFGs, showing that under certain assumptions, no polynomial-time algorithm can learn an $\eps$-CCE with less than $2^{\log_2^{1/2 - o(1)} |\cT|}$ oracle accesses to the game for even constantly large $\eps > 0$, where $|\cT|$ is the number of nodes of the EFG.
In the context of stochastic bandit, \citet{bai2022near, fiegel2023adapting} investigated online learning in EFGs with bandit feedback, establishing matched lower and upper bounds.

\section{Preliminaries}
\label{sec:preliminaries}

\paragraph{General Notation} 
We use lowercase boldface letters, such as $\xb$, to denote vectors. Let $\xb \odot \yb$ represent the element-wise product of two vectors, and $|\xb|$ the element-wise absolute value. For an index set $\cC$, denote by $\xb[\cC] \in \RR^{\cC}$ the entries of $\xb$ at indices in $\cC$, and by $|\cC|$ the set cardinality. Let $\llbracket k \rrbracket$ be the set $\{1, 2, \dots, k\}$ and $\emptyset$ the empty set. Denote the simplex over the set $\cC$ by $\Delta^\cC$. The logarithm of $x$ to base $2$ is denoted as $\log x$. For non-negative sequences $\{a_n\}$ and $\{b_n\}$, $a_n \leq \cO(b_n)$ or $b_n \geq \Omega(a_n)$ indicates the existence of a global constant $C > 0$ such that $a_n \leq Cb_n$ for all $n > 0$.

\paragraph{Extensive-Form Games} An extensive-form game (EFG) is an $n$-player game with a sequential structure that can be represented using a tree. A detailed definition of EFG is available in Appendix~\ref{sec:extended-notations}. Each node represents a game state where an agent (a.k.a. player) $i \in \llbracket n \rrbracket$ or the environment takes action. We use superscript $(i)$ to denote properties of player $i$, but also omitting the superscript when context allows. Internal nodes branch into feasible actions. At these nodes, the designated player selects an action, advancing the game to the subsequent state according to the tree. The game concludes at a terminal node $z\in \cZ$, where players receive a reward $\ub[z]$. The goal of each player is to maximize their expected reward.
We assume the reward for each player is bounded by $1$ as follows:
\begin{assumption} \label{ass:reward}
    The reward received by player $i$ at any terminal node $z \in \cZ$ satisfies $\ub^{(i)}[z] \in [0, 1]$.
\end{assumption}

\paragraph{Tree-Form Sequential Decision Process} In an EFG, an individual player $i$'s decision problem can be modeled by a tree-form sequential decision process (TFSDP). Let $\cJ$ denote the set of decision points, where each point $j \in \cJ$ corresponds to an information set in the EFG. At each decision point, the player is provided with a set of available actions $\cA_j$ and must select an action $a \in \cA_j$. After an action $a$ is taken at decision point $j$, the game either concludes before the player acts again or continues to a set of possible next decision points determined by actions of other players or by stochastic events. We denote the set of potential subsequent decision points as $\cC_{ja} \subseteq \cJ$, which are reached immediately after action $a$ at decision point $j$. The tree structure guarantees non-overlapping successors, meaning $\cC_{ja} \cap \cC_{j'a'} = \emptyset$ for any distinct pairs $ja$ and $j'a'$, where $j \neq j'$ or $a \neq a'$. This encapsulation of all past actions and outcomes at each decision point is known as \emph{perfect recall}.

The point-action pair $ja$, where action $a$ is taken at decision point $j$, is referred to as an \emph{observation point}. This leads to a new state influenced by other agents and the environment. We denote the set of all point-action pairs as $\nonrootseqs := \{ja \mid j \in \cJ, a \in \cA_j\}$. Each decision point $j \in \cJ$ has a parent $p_j$, the last observation point on the path from the root of the decision process to $j$. If no action precedes $j$, $p_j$ defaults to the special observation point $\emptyset$.
We define $\seqs := \nonrootseqs \cup \{\emptyset\}$ as the set of observation points, each also called a \emph{sequence}. The total set of points in the TFSDP, $\tfsdp := \cJ \cup \seqs$, includes both decision points and sequences. We use $h \in \tfsdp$ for unspecified point types. The TFSDP concludes at terminal observation points $\leaveseqs = \{\sigma \in \seqs : \cC_\sigma = \emptyset\}$, where reward $\rb[\sigma]$ is observed. Under Assumption~\ref{ass:reward}, it holds that $\rb[\sigma] \in [0, 1]$.

\paragraph{Strategies and Transition Kernels} A strategy profile for a player in a TFSDP is an assignment of probability distributions over actions $\cA_j$ at each decision point $j \in \cJ$. 
As customary when using convex optimization techniques in EFGs, we represent strategies in the \emph{sequence-form representation} \citep{von1996efficient}. This representation stores a player's strategy as a vector whose entries represent the probability of \emph{all} of the player's actions on the \emph{path} from the root to the points. Since products of probabilities on paths are stored directly as variables, expected utilities are \emph{multilinear} in the sequence-form representation of the players' strategies. For symmetry reasons which will become apparent later, we slightly depart from the typical definition of the sequence form, by storing the product of a player's action probabilities on paths from the root to \emph{all} points in the tree---not only those that belong to the player. We call this representation the \emph{extended sequence-form representation}. 
For an extended sequence-form strategy to be valid, probability conservation constraints must be satisfied at every point of the tree. Specifically, the set of all valid extended sequence-form strategies is given by
\begin{align*}
    \hat{\cQ} &:= \left\{ \xb \in [0, 1]^{\tfsdp}:
    \begin{aligned}
    &\xb[\emptyset] = 1 \\
    &\xb[\sigma] = \xb[j] & \quad \forall \sigma \in \nonleaveseqs, j \in \cC_\sigma\\
    &\xb[j] = \sum_{a \in \cA_j} \xb[ja]  &\quad \forall j \in \cJ
    \end{aligned}
    \right\},
\end{align*}
The distribution of observation outcomes at each observation point, or the transition kernel, is determined by the strategy played by other agents as well as the environment. 
It can viewed as an opponent who only acts at the observation points. This allows us to encode the transition kernel using a vector $\yb \in [0, 1]^{\tfsdp}$ similar to the sequence-form strategy.
The entry corresponding to each point $h \in \tfsdp$ represents the product of transition probabilities on the path from the root to the point. Formally, the transition kernel space is given by
\begin{align*}
    \hat{\cY} &:= \left\{ \yb \in [0, 1]^{\tfsdp}:
    \begin{aligned}
    &\yb[\emptyset] = 1 \\
    &\yb[\sigma] = \sum_{j \in \cC_\sigma} \yb[j] &\quad  \forall \sigma \in \nonleaveseqs\\
    &\yb[j] = \yb[ja]  &\quad \forall j \in \cJ, a \in \cA_j \\
    \end{aligned}
    \right\},
\end{align*}

We parameterize the vector spaces primarily over the terminals, since the reach of internal points can be uniquely determined by terminal reaches. We define the compressed extensive-form decision space $\cQ := \{\xb[\leaveseqs] \mid \xb \in \hat{\cQ}\}$ and the compressed transition kernel $\cY := \{\yb[\leaveseqs] \mid \yb \in \hat{\cY}\}$, representing the projection of the corresponding spaces onto the vector space generated by terminal observation points.
The existing of one-to-one mapping guarantees that $\cQ$ and $\cY$ are homogeneous to $\hat{\cQ}$ and $\hat{\cY}$. For each non-terminal point $h \in \tfsdp \setminus \leaveseqs$, we denote by $\xb[h]$ the value of $\hat{\xb}[h]$ in $\hat{\cQ}$, corresponding to the compressed strategy profile $\xb$. 
When the agent adopts strategy profile $\xb \in \cQ$ while the transition kernel aligns with $\yb \in \cY$, the reach probability of terminal point $\sigma \in \leaveseqs$ is given by $\xb[\sigma]\yb[\sigma]$. The expected reward of the player can be computed from $u(\xb; \wb) = \langle \xb, \wb \rangle$, where $\wb := \rb \odot \yb$ is the reward vector, or the gradient of utility.

To assess the complexity of the game, we use several complexity measures for the extensive-form decision space. We define the \emph{tree size} and \emph{leaf count}, denoted as $\Qsize$ and $\Qwidth$, as the maximum number of observation points and terminal observation points that can be reached among all pure strategy profiles, respectively. 
Formally, we write 
\begin{align*}
    \Qsize := \sup_{\xb \in \cQ} \|\xb[\seqs]\|_1 = \sup_{\xb \in \cQ} \sum_{\sigma \in \seqs} \xb[\sigma], \qquad 
    \Qwidth := \sup_{\xb \in \cQ} \|\xb\|_1 = \sup_{\xb \in \cQ} \sum_{\sigma \in \leaveseqs} \xb[\sigma],
\end{align*}
where we implicitly extend the domain of $\xb$ to $\hat{\cQ}$ when writing $\xb[\seqs]$.
We further define $\cV := \cQ \cap \{0, 1\}^{\leaveseqs}$ as the vertices in the extensive-form decision space. Each vertex refers to a pure strategy profile of the player, which reduced to a norm-form strategy. The number of reduced normal-form strategy is given by $|\cV|$. We remark that both $\Qsize$ and $|\cV|$ have been used in the literature \citep[e.g.,][]{farina2022kernelized}.



\paragraph{Subtree} As we will often incorporate the recursive structure in TFSDP, we define a \emph{subtree} as the subgame starting from some internal point $h \in \tfsdp$. For two points $h, h' \in \tfsdp$, we write $h' \succeq h$ if $h'$ is reachable from $h$ in TFSDP.
Let $\tfsdp_h = \{h' \in \tfsdp: h' \succeq h\}$ and $\leaveseqs_h = \{\sigma \in \leaveseqs: \sigma \succeq h\}$ be the sets of points and terminals reachable from $h$. We denote by $\cQ_{h}$ and $\cY_{h}$ the projected spaces of $\cQ$ and $\cY$ over $[0, 1]^{\leaveseqs_h}$, with restrictions $\xb[h] = 1$ and $\yb[h] = 1$, respectively.
Formally, for any point $h \in \tfsdp$, we define the compressed projected decision space as $\cQ_{h} := \{\xb[\leaveseqs_h] \mid \xb \in \cQ, \xb[h] = 1\}$. From the definition of the compressed extensive-form decision space, this space can be seen as a projection of
\begin{align*}
    \hat{\cQ}_h &:= \left\{ \xb \in [0, 1]^{\tfsdp_h}:
    \begin{aligned}
    &\xb[h] = 1 \\
    &\xb[\sigma] = \xb[j] & \quad \forall \sigma \in \nonleaveseqs, j \in \cC_\sigma\\
    &\xb[j] = \sum_{a \in \cA_j} \xb[ja]  &\quad \forall j \in \cJ
    \end{aligned}
    \right\}.
\end{align*}
Similarly, we define the compressed projected transition kernel space as $\cY_{h} := \{\yb[\leaveseqs_h] \mid \yb \in \cY, \yb[h] = 1\}$. It is important to note that this space exhibits a similar closed form to that of the compressed extensive-form decision space.


\paragraph{Proximal Methods}

We review the standard objects and notations that relate to proximal methods. For a given decision set $\cQ$, the proximal method requires a distance generating function (DGF) $\varphi: \cQ \rightarrow \RR$ defined on the decision set.
The algorithm is valid when the DGF is $\mu$-strongly convex with respect to some norm $\|\cdot\|$. The DGF induces a generalized notion of distance $\cD_{\varphi}: \cQ \times \cQ \rightarrow \RR_{\geq 0}$, referred to as the \emph{Bregman Divergence}, which is defined by
\begin{align*}
    \cD_{\varphi} (\hat\xb \| \xb) := \varphi(\hat\xb) - \varphi(\xb) - \langle \nabla \varphi(\xb), \hat\xb - \xb \rangle.
\end{align*}
We define the \emph{proximal operator} with respect to the feasible space $\cQ$ and the DGF $\varphi$. Given a pivot point $\xb$ and a gradient vector $\gb \in \RR^\leaveseqs$, the proximal operator $\Pi_{\varphi}(\gb, \xb)$ generalizes the notion of a gradient ascent step, and is defined as
\begin{align*}
    \Pi_{\varphi}(\gb, \xb) := \argmax_{\hat \xb \in \cQ} \{\langle \gb, \hat \xb \rangle - \cD_{\varphi}(\hat \xb \| \xb) \}.
\end{align*}
For the extensive-form decision space $\cQ$, the DGF is usually restricted to the dilated DGF \citep{hoda2010smoothing} so that the proximal operator can be efficiently computed. Moreover, it is well known that the proximal operator is Lipschitz continuous: \citep[e.g.][Lemma 2.1]{nemirovski2004prox}
\begin{lemma} \label{lm:proximal-operator-lip}
    For any $\gb, \gb' \in \RR^\leaveseqs$, it satisfies that 
    $\|\Pi_{\varphi}(\gb, \xb) - \Pi_{\varphi}(\gb', \xb)\| \leq \mu^{-1}\|\gb - \gb'\|_*$.
\end{lemma}



\section{Primal-Dual Treeplex Norms}


We first introduce the \tpLOneNorm~$\|\cdot\|\Ynorm$ and the \tpLInfNorm~$\|\cdot\|\Qnorm$, which are a primal-dual norm pair defined over the vector space $\RR^{\leaveseqs}$, with respect to a given TFSDP with the point set $\tfsdp$. As we will show later, these norms enable a better framework for analyzing FOMs in EFGs. Specifically, in the analysis of OMD, we use the fact that the $\ell_\infty$ norm for any feasible reward vector $\wb$ satisfies $\|\wb\|_\infty \leq 1$. Although \tpLInfNorm~is a relaxation of the $\ell_\infty$ norm, it can still preserve the same guarantee such that $\|\wb\|\Qnorm \leq 1$. With the relaxation, we have that the \tpLOneNorm~generates a smaller distance compared to the $\ell_1$ norm, which allows us to provide a better strong convexity modulus for the regularizer, finally improving the induced regret upper bound.

Both treeplex norms are defined as the support functions with respect to the vector of element-wise absolute values. Specifically, the support function of \tpLOneNorm~is defined using the transition kernel space $\cY$, while the support function of \tpLInfNorm~is defined using the extensive-form decision space $\cQ$. 
Formally, for some vector $\ub \in \RR^{\leaveseqs}$, we denote
\begin{align*}
    \|\ub\|\Ynorm &:= \sup_{\yb \in \cY} \langle |\ub|, \yb \rangle = \sup_{\yb \in \cY} \sum_{\sigma \in \leaveseqs} |\ub[\sigma]| \cdot \yb[\sigma], \\
    \|\ub\|\Qnorm &:= \sup_{\xb \in \cQ} \langle |\ub|, \xb \rangle = \sup_{\xb \in \cQ} \sum_{\sigma \in \leaveseqs} |\ub[\sigma]| \cdot \xb[\sigma].
\end{align*}
We remark that the \tpLInfNorm~has been used by \citet{zhang2024efficient} for analyzing low-degree swap regret minimization.
When the EFG degenerates to an NFG (normal-form game), i.e., $|\cJ| = 1$, the extensive-form decision space $\cQ$ in the TFSDP becomes a simplex $\Delta^{\leaveseqs}$, and the transition kernel $\cY = \{\one\}$ only contains the all-one vector.
It follows that the \tpLOneNorm~$\|\cdot\|\Ynorm$ and the \tpLInfNorm~$\|\cdot\|\Qnorm$ degenerate to the conventional $\ell_1$ norm and $\ell_{\infty}$ norm for the vector space, respectively.
The following lemma verifies that both \tpLOneNorm~and \tpLInfNorm~are norms in the technical sense.
The missing proofs in this section are provided in Appendix~\ref{sec:proof-to-tp-norm}.
\begin{Lemma}{lm:norm-is-norm}
    The functions $\|\cdot\|\Ynorm$ and $\|\cdot\|\Qnorm$ are norms defined on the space $\RR^\leaveseqs$. 
\end{Lemma}
Thanks to the recursive structure of TFSDP, the maximization among $\cY$ or $\cQ$ in the treeplex norms can be decomposed at each point $h \in \tfsdp$. This decomposition allows us to compute both treeplex norms in a recursive manner.
\begin{Lemma}{lm:tpnorm-recursive}
    Let $\ub \in \RR^{\leaveseqs_h}$ be a vector with respect to some point $h \in \tfsdp$. The \tpLOneNorm~and the \tpLInfNorm~of vector $\ub$ over $\tfsdp_h$ can be computed recursively as follows.
    \begin{itemize}[nosep]
        \item  If $h = \sigma \in \leaveseqs$ is a terminal observation point, then:
        \[\textstyle{
            \|\ub\|\Ynorms{\sigma} := |\ub[\sigma]|, \qquad \|\ub\|\Qnorms{\sigma} := |\ub[\sigma]|.
        }\]
        \item  If $h = j \in \cJ$ is a decision point, then:
        \[\textstyle{
            \|\ub\|\Ynorms{j} := \sum_{a \in \cA_j} \|\ub[\leaveseqs_{ja}]\|\Ynorms{ja}, \qquad \|\ub\|\Qnorms{j} := \max_{a \in \cA_j} \|\ub[\leaveseqs_{ja}]\|\Qnorms{ja}.
        }\]
    \item  If $h = \sigma \in \nonleaveseqs$ is a non-terminal observation point, then: 
        \[\textstyle{
            \|\ub\|\Ynorms{j} := \max_{a \in \cA_j} \|\ub[\leaveseqs_{ja}]\|\Ynorms{ja}, \qquad \|\ub\|\Qnorms{j} := \sum_{a \in \cA_j} \|\ub[\leaveseqs_{ja}]\|\Qnorms{ja}.
        }\]
    \end{itemize}
\end{Lemma}


Equipped with the recursive formula, we are able to show that the two treeplex norms with respect to the same TFSDP are a pair of primal-dual norms.
\begin{Theorem}{thm:tpnorms-duality}
    We have $\|\cdot\|\Ynorm$ and $\|\cdot\|\Qnorm$ is a pair of primal-dual norms for a given TFSDP with point set $\tfsdp$. Specifically, for any vector $\ub \in \RR^\leaveseqs$,
    \begin{align*}
        \|\ub\|\Ynorm^* := \sup_{\vb \in \RR^\leaveseqs} \frac{\langle \ub, \vb \rangle}{\|\vb\|\Ynorm} = \|\ub\|\Qnorm.
    \end{align*}
\end{Theorem}
Furthermore, the recursive formula also enables us to bound the treeplex norms for specific vectors.
\begin{Lemma}{lm:ub-Qnorm-upperbound}
    We have $\|\xb\|\Ynorm = 1$ for any strategy profile $\xb \in \cQ$, $\|\yb\|\Qnorm = 1$ for any transition kernel $\yb \in \cY$, and $\|\wb\|\Qnorm \leq 1$ for any feasible reward vector $\wb$ under Assumption~\ref{ass:reward}.
\end{Lemma}

\section{Metric Properties of the DilEnt Regularizer and Improved Regret Bounds}
\label{sec:upper}

In this section, we study the strong convexity modulus of the weight-one dilated entropy (DilEnt) function with respect to the treeplex norms defined above. 
The DilEnt regularizer is an instantiation of the more general dilated DGFs framework \citep{hoda2010smoothing}. Specifically, a dilated DGF for an extensive-form decision space is constructed by taking a weighted sum over suitable \emph{local} regularizers $\varphi_j$ for each $j \in \cJ$, and is of the form
\begin{align*}
    \varphi: \cQ \ni \xb \mapsto \sum_{j \in \cJ} \alpha_j \varphi_j^{\square} (\xb[p_j], \{\xb[ja]\}_{a \in \cA_j}),
\end{align*}
where
\begin{align*}
    \varphi_j^{\square} (\xb[p_j], \{\xb[ja]\}_{a \in \cA_j}) := \begin{cases}
        0 & \textrm{if } \xb[p_j] = 0 \\
        \xb[p_j] \varphi_j\Big(\frac{\{\xb[ja]\}_{a \in \cA_j}}{\xb[p_j]}\Big) & \textrm{otherwise}
    \end{cases}
\end{align*}
and $\alpha_j > 0$ are flexible weight terms that can be chosen to ensure good properties. Note that we have implicitly extended the domain of $\xb$ to $\hat{\cQ}$.
Each local function $\varphi_j : \Delta^{\cA_j} \rightarrow \RR$ is required to be continuously differentiable and strongly convex on the relative interior of the local probability simplex $\Delta^{\cA_j}$. They show that the proximal steps on the dilated DGF can be efficiently computed, provided that the proximal steps for each individual $\varphi_j$ can be efficiently computed.

The DilEnt regularizer $\DilEnt: \cQ \rightarrow \RR$ is a specific instantiation of the dilated DGF with $\alpha_j = 1$ and each local regularizer $d_j$ being the negative entropy function. It has been used as a specific instantiation for practical implementations \citep[e.g.][]{lee2021last}. The function has the following closed form.
\begin{align*}
    \DilEnt : \xb \mapsto \sum_{j \in \cJ} \sum_{a \in \cA_j} \xb[ja] \ln \Big(\frac{\xb[ja]}{\xb[p_j]}\Big).
\end{align*}
Prior to our work, weighted variants of the dilated entropy had been the only variants known to have concrete strong convexity bounds, all with weights that grew with the size of the decision space beneath a given decision point~\citep{kroer2020faster,farina2021better}. These results used the standard $\ell_1$ norm as the corresponding norm for showing strong convexity.
With the help of our new primal-dual treeplex norms, we can show that the DilEnt regularizer enjoys very strong properties on the extensive-form decision space. 
We inspect the following $\yb$-weighted dilated entropy for $\yb \in \cY$:
\begin{align*}
    \DilEntw{\yb} : \xb \mapsto \sum_{j \in \cJ} \sum_{a \in \cA_j} \yb[ja] \xb[ja] \ln \Big(\frac{\xb[ja]}{\xb[p_j]}\Big).
\end{align*}
By showing that the function is equivalent to the $\yb$-weighted negative entropy on the terminal reach, we are able to prove $\DilEntw{\yb}$ is $1$-strongly convex with respect to the $\yb$-weighted $\ell_1$ norm. Since the difference $\DilEnt - \DilEntw{\yb}$ is a summation of convex functions, it can be finally demonstrated that the DilEnt regularizer is $1$-strongly convex with respect to the \tpLOneNorm. 
\begin{Lemma}{lm:strongly-convex}
    The weight-one dilated entropy (DilEnt) is $1$-strongly convex within the extensive-form decision space $\cQ$ with respect to the \tpLOneNorm~$\|\cdot\|\Ynorm$.
    Specifically, for any vector $\zb \in \RR^\leaveseqs$ and strategy profile $\xb \in \cQ$, we have that $\|\zb\|_{\nabla^2 \DilEnt(\xb)}^2 \geq \|\zb\|\Ynorm^2$.
\end{Lemma}

The complete proof is provided in Appendix~\ref{sec:proof-to-ub}.
Next, we inspect the diameter of the decision space measured by DilEnt. Using an induction statement, we show that $\ln|\cV| \leq \DilEnt(\xb) \leq 0$ holds for any strategy profile $\xb \in \cQ$. By choosing the initial strategy that minimizes the DilEnt regularizer, we upper bound the diameter of the sequence-form decision space with respect to the DilEnt regularizer.
\begin{Lemma}{lm:ub-DilEnt-div-upperbound}
    Let $\xb_1 := \argmin_{\xb \in \cQ} \DilEnt(\xb)$ be the strategy profile minimize the DilEnt regularizer. The Bregman divergence generated by the DilEnt regularizer between $\xb_1$ and any $\xb_* \in \cQ$ can be upper bounded by $\cD_\DilEnt(\xb_*, \xb_1) \leq \ln |\cV|$.
\end{Lemma}
By combining the above two lemmas, we have that the DilEnt regularizer achieves $|\cD|/\mu \leq \ln |\cV|$. Using this result, we can establish performance guarantees for FOMs with the DilEnt regularizer. We list these results in the following sections.


\subsection{Results on Online Mirror Descent}

We first inspect online learning in TFSDP with full-information feedback. Consider the use of (Predictive) OMD \citep{chiang2012online}. The pseudocode of the algorithm can be found in Algorithm~\ref{alg:OMD} in Appendix~\ref{sec:alg}. The algorithm starts from $\tilde \xb_1 \leftarrow \argmin_{\xb \in \cQ} \varphi(\xb)$ and follows a straightforward structure in each episode $t$: Take a proximal gradient step from $\tilde \xb_t$ according to the prediction $\mb_t$ to get the policy $\xb_t$; Execute policy $\xb_t$; Take another proximal gradient step from $\tilde \xb_t$ according to the observed reward vector $\wb_t$ to get $\tilde \xb_{t+1}$. The algorithm takes some DGF $\varphi$ to execute proximal steps:
\begin{align*}
\xb_t \leftarrow \Pi_{\varphi}(\eta \mb_t, \tilde \xb_t), \quad 
\tilde \xb_{t+1} \leftarrow \Pi_{\varphi}(\eta \wb_t, \tilde \xb_t).
\end{align*}
The value of prediction $\mb_t$ depends on the specific variant used (e.g. $\mb_t \leftarrow \wb_{t-1}$ in Optimistic OMD).
For the non-predictive variant, we set $\mb_t \leftarrow \zero$, and thus $\xb_t = \tilde \xb_t$. It is known that the algorithm has the following regret bound with respect to a given pair of primal-dual norms.

\begin{theorem}[Regret Bound for (Predictive) OMD, \citet{rakhlin2013online, syrgkanis2015fast}]
    \label{thm:omd-regret-ub-general}
    Let $\|\cdot\|$ and $\|\cdot\|_*$ be a pair of primal-dual norm defined on $\RR^\leaveseqs$. Let $\varphi$ be a DGF that is $\mu$-strongly convex on $\|\cdot\|$. Denote $\wb_t$ as the reward gradient received in episode $t$. 
    The cumulative regret of running (Predictive) OMD with DGF $\varphi$ and learning rate $\eta$ can be upper bounded by 
    \begin{align*}
        \Regret(T) &:= \max_{\xb_* \in \cQ}\sum_{t=1}^{T} \langle \xb_* - \xb_t, \wb_t\rangle \leq \frac{1}{\eta} \cD_\varphi(\xb_*, \xb_1) + \frac{\eta}{2\mu}\sum_{t=1}^{T} \|\wb_t - \mb_t\|_*^2.
    \end{align*}
\end{theorem}

Consider using non-predictive OMD with the DilEnt regularizer $\DilEnt$. The performance of the algorithm can be analyzed by selecting the \tpLOneNorm~and the \tpLInfNorm~as the desired pair of primal-dual norms. Using the diameter-to-strong-convexity ratio of DilEnt, we can immediately get a regret upper bound that recovers the state-of-the-art result given by KOMWU \citep{farina2022kernelized}.
\begin{Theorem}{thm:omd-regret-ub}
    Let $\varphi$ be a regularizer for extensive-form decision space $\cQ$ which is $\mu$-strongly convex on $\|\cdot\|\Ynorm$ and has a diameter $|\cD| := \sup_{x_* \in \cQ} \cD_\varphi(x_*, x_1)$. 
    Under Assumption~\ref{ass:reward}, 
    the cumulative regret of running OMD with regularizer $\varphi$ and learning rate $\eta := \sqrt{2 |\cD| / (\mu T)}$ is upper bounded by 
    \begin{align*}
        \Regret(T) \leq \sqrt{2|\cD| / \mu}\sqrt{T}.
    \end{align*}
    Moreover, if we use the DilEnt regularizer $\DilEnt$ in proximal steps, the result can be specified as
    \begin{align*}
        \Regret(T) \leq \sqrt{2\ln |\cV|} \sqrt{T}.
    \end{align*}
\end{Theorem}

\subsection{Results on Clairvoyant Online Mirror Descent}

Consider equilibrium computation in $n$-player EFGs. In this scenario, a group of agents aim to jointly learn the coarse correlated equilibrium (CCE) given only oracle access to the game (See Appendix~\ref{sec:extended-notations} for detailed definition).
We adopt Clairvoyant OMD to compute CCE, introduced by \citet{piliouras2022beyond, farina2022clairvoyant}. The pseudocode of the algorithm can be found in Algorithm~\ref{alg:COMD} in Appendix~\ref{sec:alg}.
The algorithm can be viewed as a specialized form of predictive OMD, in each episode $t \in \llbracket K \rrbracket$, an additional routines is introduced to compute the prediction vector $\mb_t$ for each player $i$ (we omit the superscript).
The prediction in the episode is calculated through $L$ steps of fixed-point iteration starting from $\xb_{t,1} \leftarrow \tilde{\xb}_{t}$. In each step $l \in \llbracket L \rrbracket$, each player $i$ computes proximal step
\begin{align*}
    \xb_{t, l+1} \leftarrow \Pi_{\varphi}(\eta \wb_{t,l}, \tilde \xb_t).
\end{align*}
where we denote by $\wb_{t, l}$ the reward vector observed by the player joint policy is corresponding to $\xb_{t,l}$. Clairvoyant OMD finally sets the prediction $\mb_t$ to the iteration result $\wb_{t,L}$. In this case, the committed policy $\xb_t \leftarrow \Pi_{\varphi}(\eta \mb_t, \tilde{\xb}_t)$ in the OMD framework is equal to $\xb_{t,L}$ and the reward vector $\wb_{t}$ is $\wb_{t,L+1}$.
We show the fixed-point iteration achieves linear convergence. Thus, the difference $\|\wb_{t,L+1} - \wb_{t,L}\|\Qnorm$ can be made as arbitrarily small. The proof starts from the following inequality, establishing that the reward vector is Lipschitz continuous with respect to the joint strategy:
\[
\|\wb^{(i)}_1 - \wb^{(i)}_2\|\Qnormp{(i)} \leq \sum_{j=1}^{n}\|\xb^{(j)}_1 - \xb^{(j)}_2\|\Ynormp{(j)}.
\]
We denote by $\wb^{(i)}_1$ the reward vector of player $i$ when all the players align with joint policy $\{\xb^{(j)}_1\}_{j=1}^{n}$.
Together with the fact that the proximal operator is Lipschitz (Lemma~\ref{lm:proximal-operator-lip}), we can show that the fixed-point iteration achieves a linear convergence rate when the learning rate $\eta$ is sufficiently small.
\begin{Lemma}{lm:COMD-fixed-point-convergence}
    Under Assumption~\ref{ass:reward}, when running COMD with DilEnt, the reward vector $\wb^{(i)}_{t, l}$ received by player $i$ in any $(t, l) \in \llbracket K \rrbracket \times \llbracket L \rrbracket$ satisfies $\|\wb^{(i)}_{t, l+1} - \wb^{(i)}_{t, l}\|\Qnormp{(i)} \leq 2(n\eta)^{l-1}$.
\end{Lemma}
Therefore, with a logarithmic number of iterations, the discrepancy between the reward vector and the prediction in the OMD framework can be made as small as $\|\wb_{t} - \mb_t\|\Qnorm = \|\wb_{t,L+1} - \wb_{t,L}\|\Qnorm \leq 1/{K}$. Substituting this result into Theorem~\ref{thm:omd-regret-ub-general} implies the average joint policy given by all $\xb_t$ among $t \in \llbracket K \rrbracket$ episodes in Clairvoyant OMD only causes a constant regret. Using the standard online-to-batch conversion \citep{cesa2006prediction}, we can demonstrate that Clairvoyant OMD finds an $\epsilon$-CCE with only a near-linear number of oracle accesses to the game, establishing the new state of the art.
\begin{Theorem}{thm:COMD-result}
    Under Assumption~\ref{ass:reward}, if every player runs Clairvoyant OMD with DilEnt regularizer and learning rate $\eta = 1/(2n)$ for $K$ episodes. With $L = \lceil \log K\rceil$ steps of inner iterations, the average joint policy $\bar \pi_{K}$ is an $\eps$-CCE for $\eps \leq \cO(n\ln |\cV| / K)$. This implies the algorithm converge to a CCE at rate $\cO(n \log |\cV| \log T / T)$ where $T = KL$ is the number of oracle access to the game.
\end{Theorem}

\section{Lower Bounds for Regret Minimization in EFGs and Optimality of DilEnt}
In this section, we show that the DilEnt regularizer has a nearly optimal diameter-to-strong-convexity ratio within the extensive-form decision space. To establish this, we prove a lower bound for online learning in TFSDP with full-information feedback. We show that every algorithm must suffer a regret lower bound that matches our regret upper bound in Theorem~\ref{thm:omd-regret-ub}. The optimality of the DilEnt regularizer is demonstrated by contradiction: If there were a regularizer with a much better diameter-to-strong-convexity ratio, then the regret of running OMD with that regularizer would violate the established regret lower bound.
We prove the lower bound by constructing a hard instance that is completely random. In this scenario, no online learning algorithm can benefit from historical data, while the cumulative reward of the optimal policy benefits from the anti-concentration properties of the maximum among random distributions.

\begin{Theorem}{thm:regret-lb}
    Given a TFSDP with decision space $\cQ$, there is an EFG satisfying Assumption \ref{ass:reward} such that: when the other players are controlled by the adversary, any algorithm \texttt{Alg} incurs an expected regret of at least $\Omega(\sqrt{\Qwidth \log |\cA_0|}\sqrt{T})$ for a given of episode number $T \geq \Qwidth$, where $|\cA_0| := \min_{j \in \cJ} |\cA_j|$ is the size of the minimum action set.
\end{Theorem}

We provide missing proofs in Appendix~\ref{sec:proof-to-lb}.
Comparing Theorem~\ref{thm:regret-lb} with Theorem~\ref{thm:omd-regret-ub}, we establish a lower bound for the diameter-to-strong-convexity ratio, $|\cD|/\mu \geq \Omega(\Qwidth \log |\cA_0|)$ for any regularizer on the extensive-form decision space with respect to our new treeplex norms. Recall the diameter-to-strong-convexity ratio of the DilEnt regularizer is at most $|\cD|/\mu \leq \ln |\cV|$, derived from combining Lemma~\ref{lm:strongly-convex} and Lemma~\ref{lm:ub-DilEnt-div-upperbound}. We establish connections between these two quantities using the following lemma, implying that the ratio achieved by the DilEnt regularizer is nearly optimal.

\begin{Lemma}{lm:ub-width-size-conv}
    Consider a TFSDP with a given point set $\tfsdp$. Define $|\cA| := \max_{j \in \cJ} |\cA_j|$ as the size of the largest action set. If there is no non-root observation point yields exactly one observation outcome, that is, $|\cC_\sigma| \geq 2$ for any $\sigma \in \nonrootseqs \setminus \leaveseqs$, then it follows that $\ln |\cV| \leq \cO(\Qwidth \log |\cA|)$. Without this structural condition, we have $\ln |\cV| \leq \cO(\Qwidth \log |\cJ \times \cA|)$ in general.
\end{Lemma}

According to Lemma~\ref{lm:ub-width-size-conv}, if every each action set has the same number of actions $|\cA_0| = |\cA|$, and no non-root observation point yields only one outcome, we have $|\cD|/\mu \leq \ln |\cV| \leq \cO(\Qwidth \log |\cA_0|)$, implying the DilEnt regularizer achieves the optimal diameter-to-strong-convexity ratio up to constant factors in this scenario. If the action sets vary in size, it creates a gap logarithmic to the size of the maximal action set. If there is some observation point that yields only one observation, the gap inflates with another factor of logarithmic to the number of decision points. All in all, the DilEnt regularizer achieves the optimal diameter-to-strong-convexity ratio up to only logarithmic factors.

\section{Conclusion, Limitations, and Open Questions}

In this paper, we introduce a new primal-dual norm pair for studying the strong convexity properties of distance-generating functions for sequence-form strategy polytopes arising in extensive-form games. Quantifying these properties is a key component in the construction of efficient first-order optimization methods for equilibrium computation. Our techniques enable us to explain the strong theoretical performance of the DilEnt regularizer, for which no meaningful strong convexity bounds were previously known. In fact, we find that among all convex regularizers for extensive-form games, DilEnt is optimal up to logarithmic factors. To establish this result, we introduced a new regret lower bound for learning in extensive-form games, which is likely of independent relevance. 


We remark that our lower bound only applies to extensive-form games with full-information feedback, a setting common in self-playing algorithms. Thus, the optimality of the DilEnt regularizer may not extend to scenarios with stochastic feedback. It would be interesting to study tight lower bounds for learning under other types of feedback. While \citet{fiegel2023adapting} gave matching lower and upper bounds under trajectory bandit feedback, results for external sampling remain open to our knowledge.


Furthermore, we can only prove tight upper and lower bounds up to constant factors for the diameter-to-strong-convexity ratio in a specific family of TFSDPs. There still remains a logarithmic gap related to the total number of sequences in general. In principle, it is possible that this gap could be further reduced, yielding a different regularizer that offers logarithmic advantages over the DilEnt regularizer. Overcoming this technical hurdle and showing that the DilEnt regularizer indeed achieves the optimal rate remains an interesting direction of research. Notably, \citet{fiegel2023adapting} were also only able to prove lower bounds in specific games under bandit feedback, alluding to the intrinsic hardness of proving lower bounds without slight restrictions to the game class.

\begin{ack}
Christian Kroer was supported by the Office of Naval Research awards N00014-22-1-2530 and N00014-23-1-2374, and the National Science Foundation awards IIS-2147361 and IIS-2238960.
\end{ack}

\clearpage

\bibliographystyle{plainnat}
\bibliography{refs.bib}


\newpage
\appendix

\section{Extended Preliminaries} \label{sec:extended-notations}
An extensive-form game (EFG) is an $n$-player game with a sequential structure representable as a tree. Let $\efg$ be the set of all nodes in the tree, where each node $z \in \efg$ corresponds to a game state. Each node is assigned to either a player $i \in \llbracket n \rrbracket$ or the environment for action. The subset $\efg^{(i)} \subseteq \efg$ comprises nodes assigned to player $i$. The environment, treated as a special player, acts according to a fixed distribution, modeling stochastic outcomes like card dealing in games. Each node's branches represent possible actions. Upon reaching a node, the assigned player selects an action, moving the game to the next node per the tree structure. Let $\efgleaves$ be the set of terminal nodes. The game concludes when it reaches a terminal node $z \in \efgleaves$, where each player $i$ receives a reward $\ub^{(i)}[z]$. Players aim to maximize their expected reward by reaching these terminal nodes. We assume $\ub^{(i)}[z] \in [0, 1]$, following Assumption~\ref{ass:reward}.

    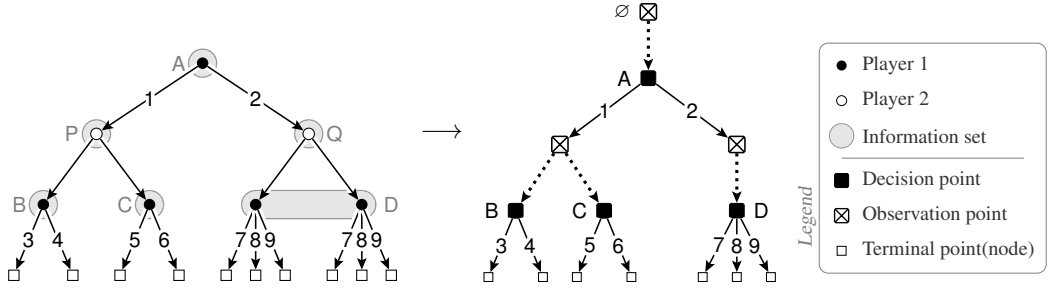
\begin{figure}[htp]\centering%
        \raisebox{1cm}{\scalebox{.98}{

\begin{tikzpicture}[baseline=0pt]
        \def\done{.9*1.6}
        \def\dtwo{.45*1.6}
        \def\dleaf{.25*1.6}
        \def\dvert{-.8*1.2}
        \contourlength{.6mm}
        \def\Xseq#1{\scalebox{.8}{\contour{white}{\seq{#1}}}}

        \node[pl1] (A) at (0, 0) {};
        \node[pl2] (X) at ($(-\done,\dvert)$) {};
        \node[pl2] (Y) at ($(\done,\dvert)$) {};
        \node[pl1] (B) at ($(X) + (-\dtwo, \dvert)$) {};
        \node[pl1] (C) at ($(X) + (\dtwo, \dvert)$) {};
        \node[termina] (l1) at ($(B) + (-\dleaf, \dvert)$) {};
        \node[termina] (l2) at ($(B) + (\dleaf, \dvert)$) {};
        \node[termina] (l3) at ($(C) + (-\dleaf, \dvert)$) {};
        \node[termina] (l4) at ($(C) + (\dleaf, \dvert)$) {};

        \node[pl1] (D1) at ($(Y) + (-\dtwo, \dvert)$) {};
        \node[pl1] (D2) at ($(Y) + (\dtwo, \dvert)$) {};
        \node[termina] (l5) at ($(D1) + (-\dleaf, \dvert)$) {};
        \node[termina] (l6) at ($(D1) + (0, \dvert)$) {};
        \node[termina] (l7) at ($(D1) + (\dleaf, \dvert)$) {};
        \node[termina] (l8) at ($(D2) + (-\dleaf, \dvert)$) {};
        \node[termina] (l9) at ($(D2) + (0, \dvert)$) {};
        \node[termina] (l10) at ($(D2) + (\dleaf, \dvert)$) {};

        \draw[action] (A)  -- node{\Xseq{1}} (X);
        \draw[action] (A)  -- node{\Xseq{2}} (Y);
        \draw[action] (B)  -- node{\Xseq{3}} (l1);
        \draw[action] (B)  -- node{\Xseq{4}} (l2);
        \draw[action] (C)  -- node{\Xseq{5}} (l3);
        \draw[action] (C)  -- node{\Xseq{6}} (l4);
        \draw[action] (D1) -- node{\Xseq{7}} (l5);
        \draw[action] (D1) -- node[fill=white,inner sep=0mm]{\Xseq{8}} (l6);
        \draw[action] (D1) -- node{\Xseq{9}} (l7);
        \draw[action] (D2) -- node{\Xseq{7}} (l8);
        \draw[action] (D2) -- node[fill=white,inner sep=0mm]{\Xseq{8}} (l9);
        \draw[action] (D2) -- node{\Xseq{9}} (l10);
        \draw[action] (X) -- (B);
        \draw[action] (X) -- (C);
        \draw[action] (Y) -- (D1);
        \draw[action] (Y) -- (D2);

        \begin{pgfonlayer}{background}
            \SingletonInfoset{X}
            \node[infoset, left=0.3mm of X] {\scalebox{.8}{\decpt{P}}};
            \SingletonInfoset{Y}
            \node[infoset, right=0.3mm of Y] {\scalebox{.8}{\decpt{Q}}};

            \SingletonInfoset{A}
            \node[infoset, left=0.3mm of A] {\scalebox{.8}{\decpt{a}}};

            \SingletonInfoset{B}
            \node[infoset, left=0.3mm of B] {\scalebox{.8}{\decpt{b}}};

            \SingletonInfoset{C}
            \node[infoset, left=0.3mm of C] {\scalebox{.8}{\decpt{c}}};

            \Infoset{(D1.center) -- (D2.center)}
            \node[infoset]  at ($(D2) + (.4, 0)$) {\scalebox{.8}{\decpt{d}}};
        \end{pgfonlayer}
    \end{tikzpicture}

        }}%
        ~$\longrightarrow$~%
        \raisebox{.8cm}{\scalebox{.98}{

        \begin{tikzpicture}[baseline=0pt,scale=1]
    \def\done{1.2}
    \def\dtwo{.6}
    \def\dleaf{.37}
    \def\dvert{-0.9}
    \contourlength{.6mm}
    \def\Xseq#1{\scalebox{.8}{\contour{white}{\seq{#1}}}}

    \node[obspt] (E) at ($(0, -\dvert)$) {};
    \node[decpt] (A) at (0, 0) {};
    \node[obspt] (X) at ($(-\done,\dvert)$) {};
    \node[obspt] (Y) at ($(\done,\dvert)$) {};
    \node[decpt] (D) at ($(Y) + (0,\dvert)$) {};
    \node[decpt] (B) at ($(X) + (-\dtwo, \dvert)$) {};
    \node[decpt] (C) at ($(X) + (\dtwo, \dvert)$) {};

    \node[draw=black,inner sep=.6mm] (l1) at ($(B) + (-\dleaf, \dvert)$) {};
    \node[draw=black,inner sep=.6mm] (l2) at ($(B) + (\dleaf, \dvert)$) {};
    \node[draw=black,inner sep=.6mm] (l3) at ($(C) + (-\dleaf, \dvert)$) {};
    \node[draw=black,inner sep=.6mm] (l4) at ($(C) + (\dleaf, \dvert)$) {};
    \node[draw=black,inner sep=.6mm] (l5) at ($(D) + (-1.25*\dleaf, \dvert)$) {};
    \node[draw=black,inner sep=.6mm] (l6) at ($(D) + (0, \dvert)$) {};
    \node[draw=black,inner sep=.6mm] (l7) at ($(D) + (1.25*\dleaf, \dvert)$) {};

    \draw[action] (A) -- node{\Xseq{1}} (X);
    \draw[action] (A) -- node{\Xseq{2}} (Y);
    \draw[action] (B) -- node{\Xseq{3}} (l1);
    \draw[action] (B) -- node{\Xseq{4}} (l2);
    \draw[action] (C) -- node{\Xseq{5}} (l3);
    \draw[action] (C) -- node{\Xseq{6}} (l4);
    \draw[action] (D) -- node{\Xseq{7}} (l5);
    \draw[action] (D) -- node[fill=white,inner sep=0mm]{\Xseq{8}} (l6);
    \draw[action] (D) -- node{\Xseq{9}} (l7);
    \draw[observ] (E) -- (A);
    \draw[observ] (X) -- (B);
    \draw[observ] (X) -- (C);
    \draw[observ] (Y) -- (D);


    \node[black, left=0mm of E] {\scalebox{.8}{\decpt{$\emptyset$}}};
    \node[black, left=0mm of A] {\scalebox{.8}{\decpt{a}}};
    \node[black, left=0mm of B] {\scalebox{.8}{\decpt{b}}};
    \node[black, left=0mm of C] {\scalebox{.8}{\decpt{c}}};
    \node[black,right=0mm of D] {\scalebox{.8}{\decpt{d}}};
\end{tikzpicture}

        }}%
        \raisebox{-1.8cm}{
        \scalebox{.85}{\begin{tikzpicture}
                \draw[gray,rounded corners] (-.35, -.35) rectangle (3.2, 3.25);
                \node[pl1]   at (0, 2.9) {};
                \node[anchor=west,black!80] at (0.2, 2.9) {\small Player 1};
                \node[pl2]   at (0, 2.35) {};
                \node[anchor=west,black!80] at (0.2, 2.35) {\small Player 2};
                \node (infoset) at (0, 1.8) {};
                \SingletonInfoset{infoset}
                \node[anchor=west,black!80] at (0.2, 1.8) {\small Information set};

                \draw[gray] (0.0,1.45) -- (2.85,1.45);
                
                \node[decpt,inner sep=1.1mm]   at (0, 1.1) {};
                \node[anchor=west,black!80] at (0.2, 1.1) {\small Decision point};
                \node[obspt]   at (0, 0.55) {};
                \node[anchor=west,black!80] at (0.2, 0.55) {\small Observation point};
                \node[termina] at (0, 0.0) {};
                \node[anchor=west,black!80] at (0.2, 0.0) {\small Terminal point(node)};
                \node[rotate=90,gray] at (-.54, 0.55) {\scalebox{.9}{\textsl{Legend}}};
            \end{tikzpicture}%
        }}

        \caption{An two-player extensive-form game and the corresponding TFSDP of player $1$. The TFSDP has decision point $\cJ = \{\textsf{A}, \textsf{B}, \textsf{C}, \textsf{D}\}$. It has tree size $\Qsize = 4$ and leaf count $\Qwidth = 2$, both given by the pure strategy $\{\textsf{A} \rightarrow \textsf{1}, \textsf{B} \rightarrow \textsf{3},\textsf{C} \rightarrow \textsf{5}\}$. Furthermore, The player $1$ has $|\cV| = 7$ pure strategy profiles in total. }

        \label{fig:efg_tfsdp}
    \end{figure}

We model imperfect information with information sets. An information set $\cI \subseteq \efg^{(i)}$ is a subset of nodes assigned to player $i$, which the player cannot distinguish. The player must act consistently across all nodes within the same information set. For example, in poker, an information set includes all states with identical public cards and bets, with each node representing a different potential hand held by the opponent. Each terminal observation point $\sigma \in \leaveseqs^{(i)}$ is associated with a set $\cI_{\sigma} \subseteq \efgleaves$ of corresponding terminal nodes in the original EFG. For each terminal node $z \in \efgleaves$, $\sigma^{(i)}_z$ denotes the observation point of player $i$ in the TFSDP.

Let $\pi = \{\xb^{(i)}\}_{i=1}^{n}$ be a joint policy of $n$ players. We denote by $u^{(i)}(\pi)$ the expected reward received by player.
Consider the terminal node $z \in \efgleaves$, the reach of probability can be computed by $\pb[z] \prod_{j=1}^{n} \xb^{(j)}[\sigma^{(j)}_z]$, where $\pb[z]$ is the product of transition probability of the environment actions from the root to $z$.
In this case, the expected reward of player $i$ is given by 
\begin{align*}
    u^{(i)}(\pi) = \sum_{z \in \cZ} \ub^{(i)}[z] \cdot \pb[z] \prod_{j=1}^{n} \xb^{(j)}[\sigma^{(j)}_z].
\end{align*}
Consider the corresponding reward vector $\wb^{(i)} := \partial_i u^{(i)}(\pi) \in \RR^{\leaveseqs^{(i)}}$ of player $i$ when the players agree on joint policy $\pi = \{\xb^{(i)}\}_{i=1}^{n}$.
The vector satisfies that $u^{(i)}(\pi) = \langle \xb^{(i)}, \wb^{(i)} \rangle$.
It is clear that the reward vector has the following closed form for each entry:
\begin{align*}
    \wb^{(i)}[\sigma] = \sum_{z \in \cI_{\sigma}} \ub^{(i)}[z] \cdot \pb[z] \prod_{j\neq i} \xb^{(j)}[\sigma^{(j)}_z] \quad \forall \sigma \in \leaveseqs^{(i)}.
\end{align*}

In this work, we examine two problems in EFGs.
For the online learning problem, an agent seeks to maximize their expected reward while facing an adversarial environment and other players in the online decision-making process. The agent can observe the reward vector $\wb_t$ after committing to the strategy profile $\xb_t$ in episode $t$.
We measure the performance of the online learning algorithm using regret. The cumulative (external) regret over $T$ episodes is defined as:
\begin{align*}
    \Regret(T) :=  \max_{\xb_* \in \cQ} \sum_{t=1}^{T} \langle \xb_*,\wb_t\rangle - \sum_{t=1}^{T} \langle \xb_t, \wb_t\rangle.
\end{align*}
where $\xb_t$ is the strategy proposed by the player in episode $t$. This definition quantifies the cumulative difference between the expected rewards that could have been obtained by the optimal strategy and those achieved under the actual strategy used.

For equilibrium computation, a group of agents aim to jointly learn the coarse correlated equilibrium (CCE) given only oracle access to the game. A joint policy profile $\pi$ is said an $\eps$-CCE, if for any player $i$, it satisfies that 
\begin{align*}
    u^{(i)}(\xb^{(i)} \times \pi^{(-i)}) \leq u^{(i)}(\pi) + \eps, \quad \forall \xb^{(i)} \in \cQ^{(i)},
\end{align*}
where we denote by $\xb^{(i)} \times \pi^{(-i)}$ the joint policy in which player $i$ takes strategy profile $\xb^{(i)}$ while other players still follows $\pi$. It is known that the equilibrium computation can be reduced to online learning: 
If every player runs a no-regret learning algorithm simultaneously, then the average joint policy $\bar{\pi}_T$ of the interaction history is an $\eps$-CCE with $\eps \leq \max_{i \in \llbracket n \rrbracket} \Regret^{(i)}(T) / T$ \citep{cesa2006prediction}.
As the players are not adversarial against each other, the regret bound of the learning algorithm can sometimes be improved compared to the online learning setting.

\section{Pseudocode of Predictive OMD and Clairvoyant OMD }\label{sec:alg}

In this section, we list the pseudocode of the algorithms used in the paper.

\begin{algorithm}[ht] \label{alg:OMD}
\caption{(Predictive) Online Mirror Descent}
$\tilde \xb_1 \leftarrow \argmin_{\xb \in \cQ} \varphi(\xb)$

\For{$t = 1$ \KwTo $T$} {
Receive prediction $\mb_{t}$ (set $\mb_{t} = \zero$ for the non-predictive variant)

$\xb_t \leftarrow \Pi_{\varphi}(\eta \mb_t, \tilde \xb_t)$

Commit policy $\xb_t$, receive reward $\langle \xb_t, \wb_t \rangle$ and observe reward vector $\wb_t$

$\tilde \xb_{t+1}\leftarrow \Pi_{\varphi}(\eta \wb_t, \tilde \xb_t)$
}
\end{algorithm}

The (Predictive) OMD framework \citep{chiang2012online, syrgkanis2015fast} depicts a family of no-regret learning algorithms. The pseudocode of the algorithm can be found in Algorithm~\ref{alg:OMD}. The algorithm takes a learning rate $\mu$ and a DGF $\varphi$ for the decision set $\cQ$ of the TFSDP as parameters. It starts from an initial point $\tilde \xb_1$ and follows a straightforward structure in each episode $t$:
\begin{itemize}[leftmargin=*]
    \item Receive prediction vector $\mb_t$ from external logic (set $\mb_t=0$ for the non-predictive variant)
    \item Take a proximal gradient step from $\tilde \xb_t$ according to the prediction $\mb_t$ to get the policy $\xb_t$.
    \item Execute policy $\xb_t$ and observe reward vector $\wb_t$.
    \item Take another proximal gradient step from $\tilde \xb_t$ according to $\wb_t$ to get $\tilde \xb_{t+1}$ for the next episode.
\end{itemize}

It can be shown from Theorem~\ref{thm:omd-regret-ub-general} that the algorithm achieves a sub-linear regret rate of $\sqrt{T}$ in general. The performance of the algorithm can be further improved by selecting more accurate $\mb_t$. For example, if we can ensure $\mb_t = \wb_t$, then the algorithm suffers only constant regret, upper bounded by $\cD_\varphi(\xb_*, \xb_1) / \eta$.

Although the reward vector $\wb_t$ in the adversarial setting is given by the environment which is generally unpredictable, in self-playing, it is determined by other agents, which generally follow a smooth dynamic and thus can be predictable. \citet{syrgkanis2015fast} introduced Optimistic OMD, which sets the prediction $\mb_t \leftarrow \wb_{t-1}$ as the reward vector given from the previous play. This allows them to establish a regret bound of $\cO(T^{1/4})$. By analyzing the higher-order derivatives, \citet{daskalakis2021near} show $\|\mb_t - \wb_t\|_*$ is small in Optimistic OMD, which finally leads to a logarithmic regret bound of $\cO(\log^4 T)$.

\begin{algorithm}[ht]
\DontPrintSemicolon
\label{alg:COMD}
\caption{Clairvoyant OMD, Decentralized}
\For {each player $i \in \llbracket n \rrbracket$, in parallel} {

$\tilde \xb_1^{(i)} \leftarrow \argmin_{\xb \in \cQ^{(i)}} \varphi^{(i)}(\xb)$

\For{$t = 1$ \KwTo $K$} { 
$\xb_{t, 1}^{(i)} \leftarrow \tilde\xb_{t}^{(i)}$

\For{$l = 1$ \KwTo $L$} { 
Synchronous with other players and commit joint policy $\pi_{t,l} \leftarrow \{\xb_{t, l}^{(i)}\}_{i=1}^n$ 

Observe reward vector $\wb_{t, l}^{(i)} \leftarrow \partial_i u^{(i)}(\pi_{t,l})$

$\xb_{t, l+1}^{(i)}\leftarrow \Pi_{\varphi^{(i)}}(\eta \wb^{(i)}_{t, l}, \tilde \xb_{t}^{(i)})$
}

$\mb_{t}^{(i)} \leftarrow \wb_{t, L}^{(i)}$

$\xb_{t}^{(i)} \leftarrow \xb_{t, L}^{(i)}; \pi_t \leftarrow \pi_{t,L}$
\tcp*{$\xb_{t}^{(i)} = \Pi_{\varphi^{(i)}}(\eta \mb^{(i)}_{t}, \tilde \xb_{t}^{(i)})$}

$\wb_{t}^{(i)} \leftarrow \wb_{t, L+1}^{(i)}$

$\tilde \xb_{t+1}^{(i)} \leftarrow \xb_{t, L+1}^{(i)}$
\tcp*{$\tilde \xb_{t+1}^{(i)} = \Pi_{\varphi^{(i)}}(\eta \wb^{(i)}_{t}, \tilde \xb_{t}^{(i)})$}

}
}

Report average joint policy $\bar \pi_{K}$ of $\pi_t$ among $t \in \llbracket K \rrbracket$
\end{algorithm}

Note that the desired prediction $\mb_t = \wb_t$ is a solution to $\wb_t^{(i)} = \partial_i u^{(i)}(\pi_t)$ for every player $i$, where the joint policy $\pi_t = \{\xb_t^{(i)}\}_{i=1}^n$ is given by proximal step $\xb_t^{(i)} =  \Pi_{\varphi^{(i)}}(\eta \wb^{(i)}_{t}, \tilde \xb_{t}^{(i)})$ for each player $i$. Note this is a fixed-point to dynamics
\begin{align*}
    \begin{cases}
        \pi_t &\leftarrow \{\xb_t^{(i)}\}_{i=1}^n \\
        \wb_t^{(i)} &\leftarrow \partial_i u^{(i)}(\pi_t) \\
        \xb_t^{(i)} &\leftarrow \Pi_{\varphi^{(i)}}(\eta \wb^{(i)}_{t}, \tilde \xb_{t}^{(i)})
    \end{cases}
\end{align*}
Clairvoyant OMD \citep{piliouras2022beyond, farina2022clairvoyant} uses these update rules to find $\mb_t$ through the fixed-point iteration. By showing that all these updating rules are Lipschitz, one can see that this fixed-point iteration achieves a linear convergence rate when the learning rate $\eta$ is small. 
Let $K$ be the number of episodes used for aggravating the average joint policy. The linear convergence rate allows one to find $\mb_t$ where $\|\mb_t - \wb_t\|_*$ is polynomially small with $\cO(1/K)$ in only $L = \cO(\log K)$ steps, indicating that the corresponding OMD dynamics only suffer from a constant regret bound. This finally establishes that the algorithm has an almost-optimal convergence rate of $\log T / T$ where $T = KL$ is the number of oracle access to the game.

\section{Proof of Treeplex Norm} \label{sec:proof-to-tp-norm}
 
\subsection{Proof of Lemma~\ref{lm:norm-is-norm}}
\restateLemma{lm:norm-is-norm}
\begin{proof}
    We verify that $\|\cdot\|\Ynorm$ and $\|\cdot\|\Qnorm$ are norms as follows:

    \textbf{Positive definiteness:} It is clear from the definition that $\|\ub\|\Ynorm = 0$ and $\|\ub\|\Qnorm = 0$ when $\ub = \zero$. When $\ub \neq \zero$, there exists $\sigma_\ub \in \leaveseqs$ such that $|\ub[\sigma_\ub]| > 0$. From the definition of $\cY$, we can always find some transition kernel $\yb_{\ub} \in \cY$ such that $\sigma_\ub$ is reachable. In this case, $\yb_{\ub}[\sigma_\ub] > 0$ and we have 
    \begin{align*}
        \|\ub\|\Ynorm \geq \sum_{\sigma \in \leaveseqs} |\ub[\sigma]| \cdot \yb_{\ub}[\sigma] \geq |\ub[\sigma_\xb]| \cdot \yb_{\ub}[\sigma_{\ub}] > 0.
    \end{align*}
    Similarly, we can always find some strategy profile $\xb_{\ub}$ with $\xb_{\ub}[\sigma_\ub] > 0$, which implies that
    \begin{align*}
        \|\ub\|\Qnorm \geq \sum_{\sigma \in \leaveseqs} |\ub[\sigma]| \cdot \xb_{\ub}[\sigma] \geq |\ub[\sigma_\xb]| \cdot \xb_{\ub}[\sigma_{\ub}] > 0.
    \end{align*}
    This verifies that both functions are strictly positive on non-zero vectors.

    \textbf{Homogeneity:} For any $k \in \RR$ and $\ub \in \RR^\leaveseqs$, it holds that
    \begin{align*}
        \|k\ub\|\Ynorm &= \sup_{\yb \in \cY} \sum_{\sigma \in \leaveseqs} |k\ub[\sigma]| \cdot \yb[\sigma] = |k|\sup_{\yb \in \cY} \sum_{\sigma \in \leaveseqs} |\ub[\sigma]| \cdot \yb[\sigma] = |k| \cdot \|\ub\|\Ynorm \\ 
        \|k\ub\|\Qnorm &= \sup_{\xb \in \cQ} \sum_{\sigma \in \leaveseqs} |k\ub[\sigma]| \cdot \xb[\sigma] = |k|\sup_{\xb \in \cQ} \sum_{\sigma \in \leaveseqs} |\ub[\sigma]| \cdot \xb[\sigma] = |k| \cdot \|\ub\|\Qnorm, \\ 
    \end{align*}
    which verifies absolute homogeneity. 

    \textbf{Triangle inequality:} We verify the triangle inequality for any $\ub_1, \ub_2 \in \RR^\leaveseqs$ by
    \begin{align*}
        \|\ub_1 + \ub_2\|\Ynorm &= \sup_{\yb \in \cY} \sum_{\sigma \in \leaveseqs} |\ub_1[\sigma] + \ub_2[\sigma]| \cdot \yb[\sigma] \\
        &\leq \sup_{\yb \in \cY} \sum_{\sigma \in \leaveseqs} |\ub_1[\sigma]| \cdot \yb[\sigma] + \sup_{\yb \in \cY} \sum_{\sigma \in \leaveseqs} |\ub_2[\sigma]| \cdot \yb[\sigma] = \|\ub_1\|\Ynorm + \|\ub_2\|\Ynorm.
    \end{align*}
    With a similar calculation, it is straightforward to check $\|\cdot\|\Qnorm$ also satisfies triangle inequality. 
    
    To conclude, both $\|\cdot\|\Ynorm$ and $\|\cdot\|\Qnorm$ are norms.
\end{proof}

\subsection{Proof of Lemma~\ref{lm:tpnorm-recursive}}

\restateLemma{lm:tpnorm-recursive}
\begin{proof}
    Consider the statement for \tpLOneNorm.
    The plan is to prove the statement by induction on $\tfsdp$ from the bottom up on TFSDP. We will show that the recursive definition ensures $\|\ub\|\Ynorms{h}$ represents the \tpLOneNorm~restricted to the subtree of $h$. Specifically, we will demonstrate that for every point $h \in \tfsdp$ and vector $\ub \in \RR^{\leaveseqs_{h}}$, the recursive formula at point $h$ matches the original definition of \tpLOneNorm~with restricted in the subtree $\tfsdp_h$, which is: 
    $$\|\ub\|\Ynorms{h} = \sup_{\yb \in \cY_{h}} \langle |\ub|, \yb \rangle.$$ 

    \textbf{Case 1:} The base case for the induction occurs when $h = \sigma$, where $\sigma \in \leaveseqs$ is a terminal point. We verify the induction basis by
    \begin{align*}
        \|\ub\|\Ynorms{\sigma} = |\ub[\sigma]| \cdot 1 = \sup_{\yb \in \cY_{\sigma}} \langle |\ub|, \yb \rangle,
    \end{align*}
    where the last equality is given by the fact that $\yb[\leaveseqs_{\sigma}] = 1$ for $\yb \in \cY_{\sigma}$.

    \textbf{Case 2:} For any decision point $h = j \in \cJ$, it holds that
    \begin{align*}
        \|\ub\|\Ynorms{j} &= \sum_{a \in \cA_j} \|\ub[\leaveseqs_{ja}]\|\Ynorms{ja} \\
        &= \sum_{a \in \cA_j} \sup_{\yb_{ja} \in \cY_{ja}} \langle |\ub[\leaveseqs_{ja}]|, \yb_{ja}[\leaveseqs_{ja}] \rangle\\
        &= \sup_{\{\yb_{ja} \in \cY_{ja}\}_{a \in \cA_j}} \sum_{a \in \cA_j} \langle |\ub[\leaveseqs_{ja}]|, \yb_{ja}[\leaveseqs_{ja}] \rangle \\
        &= \sup_{\yb \in \cY_j} \sum_{a \in \cA_j}\langle |\ub[\leaveseqs_{ja}]|, \yb[\leaveseqs_{ja}] \rangle  \\
        &= \sup_{\yb \in \cY_j} \langle |\ub|, \yb \rangle,
    \end{align*}
    where the first equality holds according to the recursive definition of $\|\ub\|\Ynorms{j}$, the second equality follows from the induction hypothesis, the third equality holds as the set of terminal points $\leaveseqs_{ja}$ are disjoint for each $a \in \cA_j$, the fourth equality follows from the fact that $\cY_j$ can be decomposed into the Cartesian product of $\cY_{ja}$ among all $a \in \cA_j$ for any $j \in \cJ$, and the last equality holds as each term of the $\ell_1$ norm is positive.

    \textbf{Case 3:} For some non-terminal observation point $h = \sigma \in \nonleaveseqs$, we will show that quantity $\|\ub\|\Ynorms{\sigma} = \max_{j \in \cC_\sigma} \|\ub[\leaveseqs_{j}]\|\Ynorms{j}$ is neither less than nor greater than $\sup_{\yb \in \cY_\sigma} \langle |\ub|, \yb[\leaveseqs_{\sigma}] \rangle$. Firstly, for decision point $j \in \cC_\sigma$ that is a successor of $\sigma$, it satisfies that 
    \begin{align*}
        \|\ub[\leaveseqs_{j}]\|\Ynorms{j} = \sup_{\yb \in \cY_j} \langle |\ub[\leaveseqs_{j}]|, \yb[\leaveseqs_{j}] \rangle \leq \sup_{\yb \in \cY_\sigma} \langle |\ub|, \yb \rangle,
    \end{align*}
    where the equality holds from the induction hypothesis and the inequality holds since the inner product can be upper bounded according to $\langle |\ub[\leaveseqs_{j}]|, \yb[\leaveseqs_{j}] \rangle \leq \langle |\ub|, \yb \rangle$ and $\cY_j$ is a subset of $\cY_\sigma$. By taking the maximal among all successor $j \in \cC_\sigma$, this inequality immediately establishes an upper bound for $\|\ub\|\Ynorms{\sigma}$:
    \begin{align} \label{eq:lm:tpnorm-recursive-case3-1}
        \|\ub\|\Ynorms{\sigma} = \max_{j \in \cC_\sigma} \|\ub[\leaveseqs_{j}]\|\Ynorms{j} \leq \sup_{\yb \in \cY_\sigma} \langle |\ub|, \yb \rangle.
    \end{align}
    Moreover, fix some transition kernel $\yb \in \cY_\sigma$. According to the tree-structure of the TFSDP, we have that the vector $\yb[\leaveseqs_j] / \yb[j]$ with respect to some successor $j \in \cC_\sigma$ is a valid transition kernel with in the subtree of $j$. In other words, we have vector $\yb[\leaveseqs_j] / \yb[j] \in \cY_j$. Thus, we can write
    \begin{align*}
        \langle |\ub|, \yb \rangle &= \sum_{j \in \cC_\sigma} \yb[j] \cdot \langle |\ub[\leaveseqs_j]|, \yb[\leaveseqs_{j}] / \yb[j] \rangle \\
        &\leq \sum_{j \in \cC_\sigma} \yb[j] \cdot \sup_{\yb_j \in \cY_j} \langle |\ub[\leaveseqs_{j}]|, \yb_j[\leaveseqs_{j}] \rangle \\
        &=\sum_{j \in \cC_\sigma} \yb[j] \cdot \|\ub[\leaveseqs_{j}]\|\Ynorms{j} \\
        &\leq \max_{j \in \cC_\sigma} \|\ub[\leaveseqs_{j}]\|\Ynorms{j}
    \end{align*}
    where the first equality follows from the fact that $\leaveseqs_j$, for all successor $j \in \cC_\sigma$, forms a partition of $\leaveseqs_{\sigma}$, the first inequality holds since $\yb[\leaveseqs_j] / \yb[j] \in \cY_j$, the second equality follows from induction hypothesis, and the last inequality holds since $\sum_{j \in \cC_\sigma} \yb[j] = 1$ and $\yb[j] \geq 0$ for transition kernel $\yb \in \cY_\sigma$. By taking supremum among all transition kernel $\yb \in \cY_\sigma$, we establishes an lower bound $\|\ub\|\Ynorms{\sigma}$
    \begin{align} \label{eq:lm:tpnorm-recursive-case3-2}
        \sup_{\yb \in \cY_\sigma} \langle |\ub|, \yb \rangle \leq \max_{j \in \cC_\sigma} \|\ub[\leaveseqs_{j}]\|\Ynorms{j} = \|\ub\|\Ynorms{\sigma}.
    \end{align}
    As the upper bound in \eqref{eq:lm:tpnorm-recursive-case3-1} and the lower bound in \eqref{eq:lm:tpnorm-recursive-case3-2} agrees on the same quantity, we immediately reach the following equation 
    \begin{align*}
        \|\ub\|\Ynorms{\sigma} = \max_{j \in \cC_\sigma} \|\ub[\leaveseqs_{j}]\|\Ynorms{j} = \sup_{\yb \in \cY_\sigma} \langle |\ub|, \yb \rangle
    \end{align*}
    which proves the induction statement on observation point $\sigma$.
    
    In general, the induction hypothesis always holds. By inspecting $h = \emptyset$, we reach the desired statement in which $\|\ub\|\Ynorms{\emptyset} = \sup_{\yb \in \cY} \langle |\ub|, \yb \rangle = \|\ub\|\Ynorm$.

    Finally, since the \tpLInfNorm~closely mirrors \tpLOneNorm, the result for \tpLInfNorm~can be directly reached with the only modification being the interchange of cases 2 and 3.
\end{proof}


\subsection{Proof of Theorem~\ref{thm:tpnorms-duality}}
\restateTheorem{thm:tpnorms-duality}
\begin{proof}
    Firstly, from the definition of the \tpLOneNorm~and the \tpLInfNorm, the norm of any vector is equal to the norm of the vector that takes the absolute value at each index, that is, we have $\|\vb\|\Ynorm = \||\vb|\|\Ynorm$ and $\|\ub\|\Qnorm = \||\ub|\|\Qnorm$ holds for any $\ub, \vb \in \RR^\leaveseqs$.
    Therefore, we can always choose $\vb$ such that it matches the sign of $\ub$ to maximize the inner product in the dual norm, which implies 
    \begin{align*}
        \|\ub\|\Ynorm^* = \sup_{\vb \in \RR^\leaveseqs} \frac{\langle \ub, \vb \rangle}{\|\vb\|\Ynorm} = \sup_{\vb \in \RR^\leaveseqs} \frac{\langle |\ub|, |\vb| \rangle}{\|\vb\|\Ynorm}
    \end{align*}
    We will use induction on $\tfsdp$ from the bottom up, demonstrating that $\|\ub\|\Ynorms{h}^* = \|\ub\|\Qnorms{h}$ always holds for every $h \in \tfsdp$. Specifically, we will show that for every $h \in \tfsdp$ and $\ub \in \RR^{\leaveseqs_h}$:
    \begin{align*}
        \sup_{\vb \in \RR^{\leaveseqs_h}} \frac{\langle |\ub|, |\vb| \rangle}{\|\vb\|\Ynorms{h}} = \|\ub\|\Qnorms{h}.
    \end{align*}

    \textbf{Case 1:} The induction basis occurs at $h = \sigma \in \leaveseqs$, where the statement can be verified from
    \begin{align*}
        \sup_{\vb \in \RR^{\leaveseqs_\sigma}} \frac{\langle |\ub|, |\vb| \rangle}{\|\vb\|\Ynorms{\sigma}} = \sup_{\vb \in \RR^{\leaveseqs_\sigma}} \frac{|\ub[\sigma]| \cdot |\vb[\sigma]|}{|\vb[\sigma]|} = |\ub[\sigma]| = \|\ub\|\Qnorms{\sigma}.
    \end{align*}
    where the first equality is given by Lemma~\ref{lm:tpnorm-recursive} and the last equality is given by Lemma~\ref{lm:tpnorm-recursive}.

    \textbf{Case 2:} Consider some decision point $h = j \in \cJ$. It satisfies that 
    \begin{align} \label{eq:thm:tpnorms-duality-case2-1}
        \sup_{\vb \in \RR^{\leaveseqs_j}} \frac{\langle |\ub|, |\vb| \rangle}{\|\vb\|\Ynorms{j}} &= \sup_{\vb \in \RR^{\leaveseqs_j}} \frac{\sum_{a \in \cA_j} \langle |\ub[\leaveseqs_{ja}]|, |\vb[\leaveseqs_{ja}]| \rangle}{\sum_{a \in \cA_j}\|\vb[\leaveseqs_{ja}]\|\Ynorms{ja}} = \sup_{\{\vb_{ja} \in \RR^{\leaveseqs_{ja}}\}_{a \in \cA_j}} \frac{\sum_{a \in \cA_j} \langle |\ub[\leaveseqs_{ja}]|, |\vb_{ja}| \rangle}{\sum_{a \in \cA_j}\|\vb_{ja}\|\Ynorms{ja}},
    \end{align}
    where the expression for the numerator in the first equality is valid since $\leaveseqs_{ja}$ for $a \in \cA_j$ is a partition of $\leaveseqs_j$, while the expression for the denominator follows from Lemma~\ref{lm:tpnorm-recursive}. 
    The fraction can be interpreted as a weighted average of $\langle |\ub[\leaveseqs_{ja}]|, |\vb_{ja}| \rangle / \|\vb_{ja}\|\Ynorms{ja}$, indicating that
    \begin{align*}
        \frac{\sum_{a \in \cA_j} \langle |\ub[\leaveseqs_{ja}]|, |\vb_{ja}| \rangle}{\sum_{a \in \cA_j}\|\vb_{ja}\|\Ynorms{ja}} \leq \max_{a \in \cA_j} \frac{\langle |\ub[\leaveseqs_{ja}]|, |\vb_{ja}| \rangle}{\|\vb_{ja}\|\Ynorms{ja}}.
    \end{align*}
    Additionally, by choosing some specific $a \in \cA_j$ and assigning $\vb_{ja'} = \zero$ for any $a' \neq a$, we have that 
    \begin{align*}
        \sup_{\{\vb_{ja} \in \RR^{\leaveseqs_{ja}}\}_{a \in \cA_j}} \frac{\sum_{a \in \cA_j} \langle |\ub[\leaveseqs_{ja}]|, |\vb_{ja}| \rangle}{\sum_{a \in \cA_j}\|\vb_{ja}\|\Ynorms{ja}} \geq \sup_{\vb_{ja} \in \RR^{\leaveseqs_{ja}}} \frac{\langle |\ub[\leaveseqs_{ja}]|, |\vb_{ja}| \rangle}{\|\vb_{ja}\|\Ynorms{ja}}.
    \end{align*}
    These inequalities define the upper and lower bounds for the same quantity, leading to the equation:
    \begin{align} \label{eq:thm:tpnorms-duality-case2-2}
        \sup_{\{\vb_{ja} \in \RR^{\leaveseqs_{ja}}\}_{a \in \cA_j}} \frac{\sum_{a \in \cA_j} \langle |\ub[\leaveseqs_{ja}]|, |\vb_{ja}| \rangle}{\sum_{a \in \cA_j}\|\vb_{ja}\|\Ynorms{ja}} = \max_{a \in \cA_j} \sup_{\vb_{ja} \in \RR^{\leaveseqs_{ja}}} \frac{\langle |\ub[\leaveseqs_{ja}]|, |\vb_{ja}| \rangle}{\|\vb_{ja}\|\Ynorms{ja}}.
    \end{align}
    Moreover, according to the induction hypothesis, we can replace the inner supremum by 
    \begin{align} \label{eq:thm:tpnorms-duality-case2-3}
        \sup_{\vb_{ja} \in \RR^{\leaveseqs_{ja}}} \frac{\langle |\ub[\leaveseqs_{ja}]|, |\vb_{ja}| \rangle}{\|\vb_{ja}\|\Ynorms{ja}} = \|\ub[\leaveseqs_{ja}]\|\Qnorms{ja}.
    \end{align}
    By combining \eqref{eq:thm:tpnorms-duality-case2-1}, \eqref{eq:thm:tpnorms-duality-case2-2}, \eqref{eq:thm:tpnorms-duality-case2-3}, and Lemma~\ref{lm:tpnorm-recursive}, we establish the induction statement on $j \in \cJ$:
    \begin{align*}
        \sup_{\vb \in \RR^{\leaveseqs_j}} \frac{\langle |\ub|, |\vb| \rangle}{\|\vb\|\Ynorms{j}} = \max_{a \in \cA_j}  \|\ub[\leaveseqs_{ja}]\|\Qnorms{ja} = \|\ub\|\Qnorms{j}.
    \end{align*}
    
    \textbf{Case 3:} When $h = \sigma \in \nonleaveseqs$, it follows from the symmetric relationship between \tpLOneNorm~and \tpLInfNorm~that, similarly to the previous arguments, we can conclude that
    \begin{align*}
        \sup_{\vb \in \RR^{\leaveseqs_\sigma}} \frac{\langle |\ub|, |\vb| \rangle}{\|\vb\|\Qnorms{\sigma}} = \|\ub\|\Ynorms{\sigma}.
    \end{align*}
    This suggests that the norm $\|\cdot\|\Qnorms{\sigma}$ is a dual norm of $\|\cdot\|\Ynorms{\sigma}$, which immediately leads to the desired statement, in which 
    \begin{align*}
        \sup_{\vb \in \RR^{\leaveseqs_\sigma}} \frac{\langle |\ub|, |\vb| \rangle}{\|\vb\|\Ynorms{\sigma}} = \|\ub\|\Qnorms{\sigma}.
    \end{align*}
\end{proof}

\subsection{Proof of Lemma~\ref{lm:ub-Qnorm-upperbound}}
\restateLemma{lm:ub-Qnorm-upperbound}
\begin{proof}
    According to the definition of $\cQ$ and $\cY$, for any element $\xb \in \cQ$ and $\yb \in \cY$, it is established that $\|\xb\| \Ynorms{h}$, $\|\yb\| \Qnorms{h}$, and their respective $\xb[h]$ and $\yb[h]$ agree on the same recursive formula across all $h \in \tfsdp$ given by Lemmas~\ref{lm:tpnorm-recursive}. Consequently, it is always true that $\|\xb\|\Ynorms{h} = \xb[h]$ and $\|\yb\|\Qnorms{h} = \yb[h]$, thereby implying $\|\xb\|\Ynorm = 1$ and $\|\yb\|\Qnorm = 1$. Finally, we have $\wb = \yb \circ \rb$ for $\rb \in [0, 1]^{\leaveseqs}$ under Assumption~\ref{ass:reward}. Thus, we have 
    \begin{align*}
        \|\wb\|\Qnorm = \sup_{\xb \in \cQ} \langle |\wb|, \xb\rangle  = \sup_{\xb \in \cQ} \langle |\rb| \odot |\yb|, \xb\rangle \leq \sup_{\xb \in \cQ} \langle |\yb|, \xb\rangle = \|\yb\|\Qnorm = 1.
    \end{align*}
    where the inequality is given by the $\xb$ is always non-negative as well as $\rb \leq 1$.
\end{proof}

\section{Proof of Regret Upper Bounds} \label{sec:proof-to-ub}

\subsection{Proof of Lemma~\ref{lm:strongly-convex}}
\restateLemma{lm:strongly-convex}
\begin{proof}
For some transition kernel $\yb \in \cY$, we define the $\yb$-weighted dilated entropy:
\begin{align*}
    \DilEntw{\yb}(\xb) := \sum_{j \in \cJ} \sum_{a \in \cA_j} \yb[ja] \xb[ja] \ln \Big(\frac{\xb[ja]}{\xb[p_j]}\Big).
\end{align*}
By decomposing the logarithm term in the summation, we get 
\begin{align} \label{eq:ub-convex-decomp}
    \DilEntw{\yb}(\xb) = \underbrace{\sum_{j \in \cJ} \sum_{a \in \cA_j} \yb[ja] \xb[ja] \ln \xb[ja]}_{\cI_1} - \underbrace{\sum_{j \in \cJ} \sum_{a \in \cA_j} \yb[ja] \xb[ja] \ln \xb[p_j]}_{\cI_2}.
\end{align}
We can rewrite the first term as
\begin{align} \label{eq:ub-convex-decomp-term1}
    \cI_1 &= \sum_{j \in \cJ} \sum_{a \in \cA_j} \yb[ja] \xb[ja] \ln \xb[ja] = \sum_{\sigma \in \nonrootseqs} \yb[\sigma] \xb[\sigma] \ln \xb[\sigma].
\end{align}
For the second term, we can further write:
\begin{align} \label{eq:ub-convex-decomp-term2}
    \cI_2 &= \sum_{j \in \cJ} \sum_{a \in \cA_j} \yb[ja] \xb[ja] \ln \xb[p_j] \notag \\ 
    &= \sum_{j \in \cJ} \yb[j] \xb[p_j] \ln \xb[p_j] \notag \\
    &= \sum_{\sigma \in \seqs\setminus\leaveseqs} \sum_{j \in \cC_\sigma} \yb[j] \xb[\sigma] \ln \xb[\sigma] \notag \\
    &= \sum_{\sigma \in \seqs\setminus\leaveseqs} \yb[\sigma] \xb[\sigma] \ln \xb[\sigma]
\end{align}
where the second equality is given by $\yb[j] = \yb[ja]$ for transition kernel $\yb \in \cY$ and $\xb[p_j] = \sum_{a \in \cA_j} \xb[ja]$ for strategy profile $\xb \in \cQ$, the third equality is derived from the fact that $\cC_\sigma$, for all $\sigma \in \seqs\setminus\leaveseqs$, forms a partition of $\cJ$, and the last equality follows from $\yb[\sigma] = \sum_{j \in \cC_{\sigma}} \yb[j]$ for any non-terminal observation point $\sigma \in \seqs\setminus\leaveseqs$ over $\yb \in \cY$.

Plugging \eqref{eq:ub-convex-decomp-term1} and \eqref{eq:ub-convex-decomp-term2} into \eqref{eq:ub-convex-decomp}, we obtain the following result, indicating that function $\DilEntw{\yb}$ can be expressed as the weighted negative entropy over all terminal observation points $\sigma \in \leaveseqs$:
\begin{align*}
    \DilEntw{\yb}(\xb) = \sum_{\sigma \in \leaveseqs} \yb[\sigma] \xb[\sigma] \ln \xb[\sigma].
\end{align*}

We are ready to show the strong convexity of weight-one dilated entropy $\DilEnt$. Consider a vector $\zb \in \RR^\leaveseqs$.
Since function $\DilEntw{\yb}$ is additively separable over all variables $\xb[\sigma]$, the Hessian matrix of $\DilEntw{\yb}(\xb)$ is diagonal. Thus, the squared norm of $\zb$ over $\nabla^2 \DilEntw{\yb}(\xb)$ can be interpreted according to 
\begin{align} \label{eq:ub-convex-inner-term1}
    \zb^{\top} \nabla^2 \DilEntw{\yb}(\xb)  \zb &= \sum_{\sigma \in \leaveseqs} \zb[\sigma]^2 \cdot \nabla^2_{\xb[\sigma]} \DilEntw{\yb}(\xb) \notag \\
    &=\sum_{\sigma \in \leaveseqs} \zb[\sigma]^2 \cdot \frac{\yb[\sigma]}{\xb[\sigma]} \notag \\
    &\geq \Big( \sum_{\sigma \in \leaveseqs} \zb[\sigma]^2 \cdot \frac{\yb[\sigma]}{\xb[\sigma]} \Big) \cdot \Big(\sum_{\sigma \in \leaveseqs} \yb[\sigma]\xb[\sigma] \Big) \notag \\
    &\geq \Big(\sum_{\sigma \in \leaveseqs} |\zb[\sigma]| \cdot \sqrt{\frac{\yb[\sigma]}{\xb[\sigma]}} \cdot \sqrt{\yb[\sigma]\xb[\sigma]}  \Big)^2 \notag \\
    &= \Big(\sum_{\sigma \in \leaveseqs} |\zb[\sigma]| \cdot \yb[\sigma] \Big)^2 = \langle|\zb|, \yb\rangle^2.
\end{align}
where the first inequality follows from Lemma~\ref{lm:ub-Qnorm-upperbound} which implies $\sum_{\sigma \in \leaveseqs} \yb[\sigma]\xb[\sigma] \leq \|\yb\|\Qnorm = 1$ as $\xb \in \cX$ and $\yb \in \cY$ and the second inequality is given by the Cauchy–Schwarz inequality.

Moreover, consider the difference between function $\DilEnt(\cdot)$ and $ \DilEntw{\yb}(\cdot)$, we have 
\begin{align*}
    \DilEntw{\one - \yb}(\xb) := \DilEnt(\xb) - \DilEntw{\yb}(\xb) = \sum_{j \in \cJ} \sum_{a \in \cA_j} (1 - \yb[ja]) \xb[ja] \ln \Big(\frac{\xb[ja]}{\xb[p_j]}\Big).
\end{align*}
From transition kernel $\yb \in \cY$, we always have $1 - \yb[ja] \geq 0$ for any $j \in \cJ$ and $a \in \cA_j$. Together with the fact that $a \ln (a/b)$ is convex for $a, b \in \RR_{\geq 0}$, we have that the difference $\DilEntw{\one - \yb}$ is a positive combination of convex functions. Thus, function $\DilEntw{\one - \yb}$ is a convex function, which implies 
\begin{align} \label{eq:ub-convex-inner-term2}
    \zb^{\top} \nabla^2 \DilEntw{\one - \yb}(\xb) \zb \geq 0.
\end{align}

Using additivity of the Hessian and inequalities  \eqref{eq:ub-convex-inner-term2} and \eqref{eq:ub-convex-inner-term1}, we obtain
\begin{align*}
    \zb^{\top} \nabla^2 \DilEnt(\xb)  \zb  = \zb^{\top} \nabla^2 \DilEntw{\yb}(\xb)  \zb + \zb^{\top} \nabla^2 \DilEntw{\one - \yb}(\xb) \zb \geq \langle|\zb|, \yb\rangle^2.
\end{align*}
By taking the supremum among all transition kernels $\yb \in \cY$, we reach the final statement  
\begin{align*}
     \|\zb\|_{\nabla^2 \DilEnt(\xb)}^2 = \zb^{\top} \nabla^2 \DilEnt(\xb) \zb \geq \max_{\yb \in \cY} \langle|\zb|, \yb\rangle^2 = \|\zb\|\Ynorm^2.
\end{align*}
This concludes that the weight-one dilated entropy $\DilEnt$ is 1-strongly convex with respect to the \tpLOneNorm~$\|\cdot\|\Ynorm$.
\end{proof}

\subsection{Proof of Lemma~\ref{lm:ub-DilEnt-div-upperbound}}

\begin{Lemma}{lm:ub-DilEnt-value-lb}
    The value of the DilEnt regularizer for some strategy profile $\xb \in \cQ$ can be bounded by $-\ln |\cV| \leq \DilEnt(\xb) \leq 0$, where $|\cV|$ is the number of reduced normal-form strategies.
\end{Lemma}
\begin{proof}
    Since the DilEnt regularizer is the weighted summation of negative entropy, and the fact that negative entropy is always non-positive, we directly get $\DilEnt(\xb) \leq 0$. We will prove $\DilEnt(\xb) \geq -\ln |\cV|$ by induction on the TFSDP from the bottom up. In specific, we will show that for any point $h \in \tfsdp$ and a corresponding strategy profile $\xb_h \in \cQ_h$, the DilEnt with retricted to $\tfsdp_h$ satisfies 
    \begin{align*}
        \DilEnts{h}(\xb_h) := \sum_{j \in \cJ_h} \sum_{a \in \cA_{j}} \xb_h[ja] \ln \Big(\frac{\xb_h[ja]}{\xb_h[p_j]}\Big) \geq -\ln |\cV_h|.
    \end{align*}

    \textbf{Case 1:} The induction basis occurs at $h = \sigma \in \leaveseqs$, where the statement holds since for any $\xb_\sigma \in \cQ_\sigma$,
    \begin{align*}
        \DilEnts{\sigma}(\xb_\sigma) = 0 = -\ln |\cV_\sigma|,
    \end{align*}
    where we have $|\cV_\sigma| = 1$ from Lemma~\ref{lm:pure-startegy-count-recursive}.

    \textbf{Case 2:} Consider some decision point $h = j \in \cJ$. We can decompose the DilEnt regularizer over $\xb_j \in \cQ_j$ according to 
    \begin{align} \label{eq:lm:ub-DilEnt-value-lb-case2-1}
        \DilEnts{j}(\xb_j) &= \sum_{a \in \cA_j} \xb_j[ja] \ln \xb_j[ja] + \sum_{a \in \cA_j}\sum_{j' \in \cJ_{ja}} \sum_{a' \in \cA_{j'}} \xb_j[j'a']\ln \Big(\frac{\xb_j[j'a']}{\xb_j[p_{j'}]}\Big) 
    \end{align}
    Consider the vector $\xb_{ja} := \xb[\cE_{ja}] / \xb[ja]$ for some action $a \in \cA_j$. If $\xb[ja] = 0$, we have 
    \begin{align*}
        \sum_{j' \in \cJ_{ja}} \sum_{a' \in \cA_{j'}} \xb_j[j'a']\ln \Big(\frac{\xb_j[j'a']}{\xb_j[p_{j'}]}\Big) = 0 \geq -\xb_j[ja]  \ln |\cV_{ja}|.
    \end{align*}
    Otherwise, it satisfies that $\xb_{ja} \in \cQ_{ja}$ according to the tree-structure of TFSDP. Thus, we can write
    \begin{align*}  
        \sum_{j' \in \cJ_{ja}} \sum_{a' \in \cA_{j'}} \xb_j[j'a']\ln \Big(\frac{\xb_j[j'a']}{\xb_j[p_{j'}]}\Big) &= \xb_j[ja] \sum_{j' \in \cJ_{ja}} \sum_{a' \in \cA_{j'}} \Big(\frac{\xb_j[j'a']}{\xb_j[ja]} \Big) \ln \Big(\frac{\xb_j[j'a']}{\xb_j[p_{j'}]}\Big) \notag \\
        &= \xb_j[ja] \sum_{j' \in \cJ_{ja}} \sum_{a' \in \cA_{j'}} \xb_{ja}[j'a']\ln \Big(\frac{\xb_{ja}[j'a']}{\xb_{ja}[p_{j'}]}\Big) \notag \\
        &= \xb_j[ja] \DilEnts{ja} (\xb_{ja}) \notag \\
        &\geq -\xb_j[ja]  \ln |\cV_{ja}|,
    \end{align*}
    where the last inequality is given by induction hypothesis.

    In general, it always satisfies that 
    \begin{align*}
        \sum_{j' \in \cJ_{ja}} \sum_{a' \in \cA_{j'}} \xb_j[j'a']\ln \Big(\frac{\xb_j[j'a']}{\xb_j[p_{j'}]}\Big) \geq -\xb_j[ja]  \ln |\cV_{ja}|.
    \end{align*}
    Plugging this inequality into \eqref{eq:lm:ub-DilEnt-value-lb-case2-1} gives
    \begin{align*}
        \DilEnts{j}(\xb_j) &\geq \sum_{a \in \cA_j} \xb_j[ja] \ln \xb_j[ja] - \sum_{a \in \cA_j} \xb_j[ja] \ln |\cV_{ja}| \\
        &\geq \ln \Big(\sum_{a \in \cA_j} \exp(-\ln |\cV_{ja}|)\Big) \\
        &= -\ln \Big(\sum_{a \in \cA_j} |\cV_{ja}|\Big) \\
        &= -\ln |\cV_j|.
    \end{align*}
    where the first inequality holds since the minimizer is given by $\xb_j[ja] \propto \exp(-\ln |\cV_{ja}|)$, and the last equality is given by Lemma~\ref{lm:pure-startegy-count-recursive}.

    \textbf{Case 3:} Consider some decision point $h = \sigma \in \seqs$. We can decompose the weight-one dilated entropy according to the tree-structure of TFSDP. Thus, we can write 
    \begin{align*}
        \DilEnts{\sigma}(\xb_\sigma) &= \sum_{j \in \cC_{\sigma}}\sum_{j' \in \cJ_{j}} \sum_{a' \in \cA_{j'}} \xb_\sigma[j'a']\ln \Big(\frac{\xb_\sigma[j'a']}{\xb_\sigma[p_{j'}]}\Big) \\
        &= \sum_{j \in \cC_{\sigma}} \DilEnts{j}(\xb_\sigma[\cE_j]) \\
        &\geq \sum_{j \in \cC_{\sigma}} -\ln |\cV_{j}| \\
        &= -\ln |\cV_{\sigma}|,
    \end{align*}
    where the second inequality is given by induction hypothesis, and the last equality is given by Lemma~\ref{lm:pure-startegy-count-recursive}.

    In general, it always holds that $\DilEnts{h}(\xb_h) \geq -\ln |\cV_{h}|$, 
    which concludes the proof. Substituting this with $h = \emptyset$ reaches the desired result.
\end{proof}

\restateLemma{lm:ub-DilEnt-div-upperbound}
\begin{proof}
    From the definition of Bregman divergence $\cD_\DilEnt$, we can write
    \begin{align} \label{eq:lm:ub-DilEnt-div-upperbound:Bregman}
        \cD_\DilEnt(\xb_*, \xb_1) &= \DilEnt(\xb_*) - \DilEnt(\xb_1) - \langle \nabla \DilEnt(\xb_1), \xb_* - \xb_1 \rangle 
    \end{align}
    According to Lemma~\ref{lm:ub-DilEnt-value-lb}, we have $\DilEnt(\xb_*) - \DilEnt(\xb_1) \leq \ln |\cV|$. From the chosen of $\xb_1$, we have $\langle \nabla \DilEnt(\xb_1), \xb_* - \xb_1 \rangle = 0$. By plugging both inequalities into \eqref{eq:lm:ub-DilEnt-div-upperbound:Bregman}, we reach the desired result 
    \begin{align*}
        \cD_\DilEnt(\xb_*, \xb_1) \leq \ln |\cV|.
    \end{align*}
\end{proof}

\subsection{Proof of Theorem~\ref{thm:omd-regret-ub}}
\restateTheorem{thm:omd-regret-ub}
\begin{proof}
    We will apply Theorem~\ref{thm:omd-regret-ub-general} with $\|\cdot\|\Ynorm$ and $\|\cdot\|\Qnorm$ be the desired primal-dual pair. In the context of the theorem, we have $\|\wb_t\|\Qnorm \leq 1$ for any $t \in \llbracket T \rrbracket$ from Lemma~\ref{lm:ub-Qnorm-upperbound}. Together with $\cD_\varphi(\xb_*, \xb_1) \leq |\cD|$, when choosing learning rate $\eta := \sqrt{2 |\cD| / (\mu T)}$, we have the regret can be upper bounded by 
    \begin{align*}
        \Regret(T) \leq \frac{1}{\eta} \cdot |\cD| + \frac{\eta}{2} \cdot T = \sqrt{2|\cD| / \mu} \sqrt{T}.
    \end{align*}
    When selecting DilEnt as the regularizer, we have $|\cD| \leq \ln |\cV|$ from Lemma~\ref{lm:ub-DilEnt-div-upperbound} and $\mu \geq 1$ from Lemma~\ref{lm:strongly-convex}. Plugging these results indicates
    \begin{align*}
        \Regret(T) \leq \sqrt{2\ln |\cV|} \sqrt{T}.
    \end{align*}
\end{proof}

\subsection{Proof of Lemma~\ref{lm:COMD-fixed-point-convergence}}

We first prove the lemma starts from the following locally Lipshitz.
\begin{lemma} \label{lm:self-play-lip-local} 
    In Algorithm~\ref{alg:COMD}, consider two joint policies that agree on all strategy profile expect for player $j \in \llbracket n \rrbracket$, $\pi_1 := \{\xb^{(1)}_0, \cdots, \xb^{(j)}_1, \cdots, \xb^{(n)}_0\}$ and $\pi_2 := \{\xb^{(1)}_0, \cdots, \xb^{(j)}_2, \cdots, \xb^{(n)}_0\}$.
    Under Assumption~\ref{ass:reward}, we have the reward vector is locally Lipschitz under the \tpLOneNorm, that is, denote by $\wb^{(i)}_1 := \partial_i  u^{(i)}(\pi_1)$ and $\wb^{(i)}_2 := \partial_i  u^{(i)}(\pi_2)$ the reward vector of player $i$, it satisfies that 
    $$\|\wb^{(i)}_1 - \wb^{(i)}_2\|\Qnormp{(i)} \leq \|\xb^{(j)}_{1} - \xb^{(j)}_{2}\|\Ynormp{(i)}.$$
\end{lemma}

\begin{proof}
    If $j = i$, we have $\wb^{(i)}_1 = \wb^{(i)}_2$ and the statement holds true from \[\|\wb^{(i)}_1 - \wb^{(i)}_2\|\Qnormp{(i)} = 0 \leq \|\xb^{(j)}_{1} - \xb^{(j)}_{2}\|\Ynormp{(i)}.\]
    Otherwise it satisfies that $j \neq i$.
    According to the definition of \tpLInfNorm, we have that 
    \begin{align} \label{eq:lm:self-play-lip-1-1}
        \|\wb^{(i)}_1 - \wb^{(i)}_2\|\Qnormp{(i)} = \sup_{\xb \in \cQ^{(i)}} \langle |\wb^{(i)}_1 - \wb^{(i)}_2|, \xb\rangle
        = \sup_{\xb \in \cQ^{(i)}} \sum_{\sigma \in \leaveseqs^{(i)}} |\wb^{(i)}_1[\sigma] - \wb^{(i)}_2[\sigma]| \cdot \xb[\sigma].
    \end{align}
    By expressing the reward vector using the strategy of other players, we have that 
    \begin{align} \label{eq:lm:self-play-lip-1-2}
        \wb^{(i)}_1[\sigma] = \sum_{z \in \cI_{\sigma}} \ub^{(i)}[z] \cdot \pb[z] \cdot \xb^{(j)}_1[\sigma_z^{(j)}] \prod_{k \neq i, j} \xb^{(k)}_0[\sigma_z^{(k)}].
    \end{align}
    Plugging \eqref{eq:lm:self-play-lip-1-2} into \eqref{eq:lm:self-play-lip-1-1} gives
    \begin{align} \label{eq:lm:self-play-lip-1-3}
        &\quad \|\wb^{(i)}_1 - \wb^{(i)}_2\|\Qnormp{(i)} \\
        &= \sup_{\xb \in \cQ^{(i)}} \sum_{\sigma \in \leaveseqs^{(i)}} \xb[\sigma]  \cdot \Big|\sum_{z \in \cI_{\sigma}} \ub^{(i)}[z] \cdot \pb[z] \cdot (\xb^{(j)}_1[\sigma_z^{(j)}] - \xb^{(j)}_2[\sigma_z^{(j)}])  \prod_{k \neq i, j} \xb^{(k)}_0[\sigma_z^{(k)}] \Big| \notag \\
        &\leq \sup_{\xb \in \cQ^{(i)}} \sum_{\sigma \in \leaveseqs^{(i)}} \xb[\sigma] \sum_{z \in \cI_{\sigma}} \pb[z] \cdot \big|\xb^{(j)}_1[\sigma_z^{(j)}] - \xb^{(j)}_2[\sigma_z^{(j)}] \big|  \prod_{k \neq i, j} \xb^{(k)}_0[\sigma_z^{(k)}]  \notag\\
        &= \sup_{\xb \in \cQ^{(i)}} \sum_{z \in \efgleaves} \xb[\sigma^{(i)}_z] \cdot \pb[z] \cdot \big|\xb^{(j)}_1[\sigma_z^{(j)}] - \xb^{(j)}_2[\sigma_z^{(j)}] \big| \prod_{k \neq i, j} \xb^{(k)}_0[\sigma_z^{(k)}],
    \end{align}
    where the inequality holds since the reward $\ub^{(i)}[z] \in [0, 1]$ and the last equality is given by permuting the summation.

    According to the definition of \tpLOneNorm, we have 
    \begin{align}  \label{eq:lm:self-play-lip-2-1}
        \|\xb^{(j)}_1 - \xb^{(j)}_2\|\Ynormp{(j)} = \sup_{\yb \in \cY^{(j)}} \langle |\xb^{(j)}_1 - \xb^{(j)}_2|, \yb\rangle 
        = \sup_{\yb \in \cY^{(j)}} \sum_{\sigma \in \leaveseqs^{(j)}} |\xb^{(j)}_1[\sigma] - \xb^{(j)}_2[\sigma]| \cdot \yb[\sigma].
    \end{align}
    By expressing the transition kernel using the strategy of other players, we can write 
    \begin{align} \label{eq:lm:self-play-lip-2-2}
        \yb[\sigma] = \sum_{z \in \cI_{\sigma}}\xb^{(i)}[\sigma^{(i)}_z] \cdot  \pb[z] \prod_{k \neq i, j} \xb^{(k)}[\sigma_z^{(k)}]
    \end{align}
    such that $\xb^{(i)} \in \cQ^{(i)}$ and $\xb^{(k)} \in \cQ^{(k)}$. Plugging \eqref{eq:lm:self-play-lip-2-2} into \eqref{eq:lm:self-play-lip-2-1} gives
    \begin{align} \label{eq:lm:self-play-lip-2-3}
        \|\xb^{(j)}_1 - \xb^{(j)}_2\|\Ynormp{(j)} &= \sup_{\{\xb^{(k)} \in \cQ^{(k)}\}_{k\neq j}} \sum_{\sigma \in \leaveseqs^{(j)}} |\xb^{(j)}_1[\sigma] - \xb^{(j)}_2[\sigma]| \sum_{z \in \cI_{\sigma}}\xb^{(i)}[\sigma^{(i)}_z] \cdot \pb[z] \prod_{k \neq i, j} \xb^{(k)}[\sigma_z^{(k)}] \notag \\
        &\geq \sup_{\xb \in \cQ^{(i)}} \sum_{\sigma \in \leaveseqs^{(j)}} |\xb^{(j)}_1[\sigma] - \xb^{(j)}_2[\sigma]| \sum_{z \in \cI_{\sigma}}\xb[\sigma^{(i)}_z] \cdot \pb[z] \prod_{k \neq i, j} \xb^{(k)}_0[\sigma_z^{(k)}] \notag \\
        &= \sup_{\xb \in \cQ^{(i)}} \sum_{z \in \efgleaves} \xb[\sigma^{(i)}_z]\cdot \pb[z] \cdot \big|\xb^{(j)}_1[\sigma_z^{(j)}] - \xb^{(j)}_2[\sigma_z^{(j)}] \big| \prod_{k \neq i, j} \xb^{(k)}_0[\sigma_z^{(k)}],
    \end{align}
    where the inequality holds as $\xb^{(k)}_0 \in \cQ^{(k)}$ and the last equality is given by permuting the summation.
    By combining \eqref{eq:lm:self-play-lip-1-3} and \eqref{eq:lm:self-play-lip-2-3}, we can reach that  
    \begin{align} \label{eq:lm:self-play-lip-3}
        \|\wb^{(i)}_1 - \wb^{(i)}_2\|\Qnormp{(i)} \leq \|\xb^{(j)}_1 - \xb^{(j)}_2\|\Ynormp{(j)}.
    \end{align}
\end{proof}

\begin{lemma} \label{lm:self-play-lip}
    In Algorithm~\ref{alg:COMD}, for two joint policy $\pi_1 = \{\xb^{(j)}_1\}_{j=1}^n$ and $\pi_2 = \{\xb^{(j)}_2\}_{j=1}^n$ where $\xb^{(j)}_1, \xb^{(j)}_2 \in \cQ^{(j)}$ are the strategy profiles for player $j \in \llbracket n \rrbracket$. Consider the corresponding reward vectors $\wb^{(i)}_1 := \partial_i u^{(i)}(\pi_1)$ and $\wb^{(i)}_2 := \partial_i u^{(i)}(\pi_2)$ of player $i \in \llbracket n \rrbracket$ when other players follow the joint policy.
    Under Assumption~\ref{ass:reward}, we have the reward vector is Lipschitz under the \tpLOneNorm, that is, 
    $$\|\wb^{(i)}_1 - \wb^{(i)}_2\|\Qnormp{(i)} \leq \sum_{j=1}^{n}\|\xb^{(j)}_1 - \xb^{(j)}_2\|\Ynormp{(j)}.$$
\end{lemma}

\begin{proof}
    Consider a series of reward vector policy $\pi_{(j)} := \{\xb^{(1)}_2, \cdots, \xb^{(j)}_2, \xb^{(j+1)}_1 \cdots, \xb^{(n)}_1\}$ which are generated by the joint policy that aligns with joint policy $\pi_2$ on the first $j$ players while aligns with $\pi_1$ on the rest.
    Denote by $\wb^{(i)}_{(j)} := \partial_i u^{(i)}(\pi_{(j)})$ the reward vector. Under this definition, it satisfies that $\wb^{(i)}_{(0)} = \wb^{(i)}_1$ and $\wb^{(i)}_{(n)} = \wb^{(i)}_2$. Therefore, we can write 
    \begin{align*}
        \|\wb^{(i)}_{1} - \wb^{(i)}_{2}\|\Qnormp{(i)} \leq \sum_{j=1}^{n} \|\wb^{(i)}_{(j-1)} - \wb^{(i)}_{(j)}\|\Qnormp{(i)} \leq \sum_{j=1}^{n} \|\xb^{(j)}_1 - \xb^{(j)}_2\|\Ynormp{(j)},
    \end{align*}
    where the first inequality follows from the triangle inequality and the second inequality is given by Lemma~\ref{lm:self-play-lip-local}.
\end{proof}

\restateLemma{lm:COMD-fixed-point-convergence}

\begin{proof}
    By applying Lemma~\ref{lm:proximal-operator-lip} to the proximal steps $\xb_{t, l}^{(i)}\leftarrow \Pi_{\DilEnt^{(i)}}(\eta \wb^{(i)}_{t, l-1}, \tilde \xb_{t}^{(i)})$ and $\xb_{t, l+1}^{(i)}\leftarrow \Pi_{\DilEnt^{(i)}}(\eta \wb^{(i)}_{t, l}, \tilde \xb_{t}^{(i)})$, we have that for every $l \geq 1$, 
    \begin{align} \label{eq:lm:COMD-fixed-point-convergence:1}
        \|\xb_{t, l+1}^{(i)} - \xb_{t, l}^{(i)}\|\Ynormp{(i)} \leq \eta \|\wb^{(i)}_{t, l}- \wb^{(i)}_{t, l-1}\|\Qnormp{(i)},
    \end{align}
    where the strongly convex modulus $\mu \geq 1$ is given by Lemma~\ref{lm:strongly-convex}. From Lemma~\ref{lm:self-play-lip}, we have that 
    \begin{align} \label{eq:lm:COMD-fixed-point-convergence:2}
        \|\wb^{(i)}_{t, l+1}- \wb^{(i)}_{t, l}\|\Qnormp{(i)} \leq \sum_{j=1}^{n} \|\xb_{t, l+1}^{(j)} - \xb_{t, l}^{(j)}\|\Ynormp{(i)}.
    \end{align}
    In addition, the difference between the initial steps can be upper bounded according to
    \begin{align} \label{eq:lm:COMD-fixed-point-convergence:3}
        \|\wb^{(i)}_{t, 2}- \wb^{(i)}_{t, 1}\|\Qnormp{(i)} \leq \|\wb^{(i)}_{t, 2}\|\Qnormp{(i)} + \|\wb^{(i)}_{t, 1}\|\Qnormp{(i)} \leq 2,
    \end{align}
    where the first inequality follows from the triangle inequality and the second inequality is given by Lemma~\ref{lm:ub-Qnorm-upperbound}.
    By combining \eqref{eq:lm:COMD-fixed-point-convergence:1}, \eqref{eq:lm:COMD-fixed-point-convergence:2}, and \eqref{eq:lm:COMD-fixed-point-convergence:3}, we reach the desired statement:
    \begin{align*}
        \|\wb^{(i)}_{t, l+1} - \wb^{(i)}_{t, l}\|\Qnormp{(i)} \leq 2(n\eta)^{l-1}.
    \end{align*}
\end{proof}

\subsection{Proof of Theorem~\ref{thm:COMD-result}}
\restateTheorem{thm:COMD-result}
\begin{proof}
    When $\eta \leq 1 / (2n)$, we have $n\eta \leq 1/2$.
    According to Lemma~\ref{lm:COMD-fixed-point-convergence}, the difference between the actual reward vector and the prediction can be upper bounded by 
    \begin{align*}
        \|\wb_{t}^{(i)} - \mb_t^{(i)}\|\Qnormp{(i)} = \|\wb_{t, L+1}^{(i)} - \wb_{t, L}^{(i)}\|\Qnormp{(i)} \leq 2(n\eta)^{L-1} \leq 2^{2-L}.
    \end{align*}
    Therefore, with $L = \lceil \log K\rceil$ steps of fixed-point iterations, the difference can be as small as $\|\wb_{t}^{(i)} - \mb_t^{(i)}\|\Qnormp{(i)} \leq 4/K$.
    
    Let $\|\cdot\|\Ynormp{(i)}$ and $\|\cdot\|\Qnormp{(i)}$ be the pair of primal-dual norms required by Theorem~\ref{thm:omd-regret-ub-general}. In the context of the theorem, we have $\|\wb_t\|\Qnormp{(i)} \leq 1$ for any $t \in \llbracket K \rrbracket$ given by the definition of \tpLInfNorm, $\cD_\DilEnt(\xb_*, \xb_1) \leq \ln |\cV|$ according to Lemma~\ref{lm:ub-DilEnt-div-upperbound}, $\mu \geq 1$ according to Lemma~\ref{lm:strongly-convex}, and $\|\wb_{t}^{(i)} - \mb_t^{(i)}\|\Qnormp{(i)} \leq 1/K$ from the reasoning above. Plugging these results into the statement of Theorem~\ref{thm:omd-regret-ub-general} gives:
    \begin{align*}
        \Regret(K) \leq \frac{1}{\eta} \cdot \ln |\cV| + \frac{\eta K}{2} \cdot \Big(\frac{4}{K}\Big)^2.
    \end{align*}
    With $\eta = 1 / (2n)$, we get that 
    \begin{align*}
        \Regret(K) \leq \cO(n\log |\cV|).
    \end{align*}

    According to the online-to-batch conversion \citep[see e.g.][]{piliouras2022beyond}, we can conclude that the average joint policy $\bar \pi_{K}$ is an $\eps$-CCE with $\eps \leq \cO(n\log |\cV| / K)$. 
    Given oracle access budget $T$, we can select $K = \lfloor T / \log T \rfloor$ satisfying $K L \leq T$. This indicates that the algorithm converge to a CCE with approximation rate $\cO(n \log |\cV| \log T / T)$.
\end{proof}

\section{Proof of Regret Lower Bounds} \label{sec:proof-to-lb}

\subsection{Proof of Theorem~\ref{thm:regret-lb}}
\restateTheorem{thm:regret-lb}

\begin{proof}
    We will prove the statement by induction on $\tfsdp$ from the bottom up. For each point $h \in \tfsdp \setminus \leaveseqs$, we construct a hard instance $\cI_h(T)$ such that any algorithm $\tfsdp$ playing in subgame $\tfsdp_h$ incurs an expected regret of at least $\Omega(\sqrt{\Qwidths{h} \log |\cA_0|}\sqrt{T})$. We construct hard instance $\cI_h(T)$ to be a two-player perfect-information EFG that has the same structure as the given TFSDP, where all the observation points are the decision points of an adversarial opponent. It can be represented using a set of random variables $\{\yb_{h,t}, \rb_{h,t}\}_{t=1}^T$, where $\yb_{h,t} \in \cY_h$ is the transition kernel and $\rb_{h,t} \in [0, 1]^{\leaveseqs_h}$ encodes the expected reward of player conditional on each terminal observation point. Note that the transition kernel $\yb_{h, t}$ is also the strategy profile of the opponent in his extensive-form decision space. In this case, the expected reward of playing strategy profile $\xb_{h,t}$ on $\tfsdp_h$ on episode $t$ can be computed by $\langle \xb_{h,t}, \wb_{h,t}\rangle$ where $\wb_{h,t} :=  \rb_{h,t}\odot \yb_{h,t}$ is the reward vector.

    \textbf{Case 1:} The base case of the induction is that $h = j \in \cJ$ and $ja \in \leaveseqs$ holds for every $a \in \cA_j$. In this case, the TFSDP with point set $\tfsdp_h$ is equivalent to a full-information multi-arm bandit problem (a.k.a. expert problem). We construct the hard instance $\cI_j(T)$ by assigning 
    \begin{align*}
        \yb_{j,t}[ja] = 1, \rb_{j,t}[ja] = \Unif(\{0, 1\})
    \end{align*}
    for every episode $t \in \llbracket T \rrbracket$, where $\Unif(\{0, 1\})$ is the Bernoulli random variable with $p = 0.5$. In this case, the entries of the reward vector $\wb_{j,t} :=  \rb_{j,t}\odot \yb_{j,t}$ are independently random variables. Thus, the expected cumulative reward among $T$ episodes of any algorithm \texttt{Alg} can be computed by 
    \begin{align} \label{eq:lm:regret-lb:eq1-1}
        \Expt\Big[\sum_{t=1}^T \langle \xb_{j,t}, \wb_{j,t}\rangle\Big] = \sum_{t=1}^{T}  \frac{1}{2} = \frac{T}{2}.
    \end{align}
    Additionally, the cumulative reward of the optimal policy can be computed according to 
    \begin{align} \label{eq:lm:regret-lb:eq1-2}
        \Expt\Big[\max_{a \in \cA_j} \sum_{t=1}^T \langle \eb_{ja}, \wb_{j,t}\rangle\Big] &= \Expt\Big[\max_{a \in \cA_j} \sum_{t=1}^T \Big(\frac{1}{2}\sigma_{a,t} + \frac{1}{2} \Big)\Big] \notag \\
        &= \frac{1}{2}\Expt\Big[\max_{a \in \cA_j} \sum_{t=1}^T \sigma_{a,t}\Big] + \frac{T}{2} \notag \\
        &\geq \Omega(\sqrt{T \log |\cA_j|}) + \frac{T}{2},
    \end{align}
    where $\eb_{ja}$ is the pure strategy profile that always executes action $a$ at decision point $j$ and $\sigma_{a,t} \sim \Unif(\{-1, 1\})$ are independent Rademacher random variables for $a \in \cA_j$ and $t \in \llbracket T\rrbracket$. The last inequality follows from Lemma A.11 in \citet{cesa2006prediction}. Combining \eqref{eq:lm:regret-lb:eq1-1} and \eqref{eq:lm:regret-lb:eq1-2} indicates that any algorithm suffer an expected regret of at least 
    \begin{align*}
        \Regret_j(T) = \Expt\Big[\max_{a \in \cA_j} \sum_{t=1}^T \langle \eb_{ja}, \wb_{j,t}\rangle\Big] - \Expt\Big[\sum_{t=1}^T \langle \xb_{j,t}, \wb_{j,t}\rangle\Big] \geq \Omega \Big(\sqrt{\log |\cA_j|}\sqrt{T}\Big).
    \end{align*}
    Note that for the $j \in \cJ$, we have $\Qwidths{j} = 1$ from Lemma~\ref{lm:qwidth-recusrive}. Together with $|\cA_j| \geq |\cA_0|$ from the definition of $|\cA_0|$, we conclude that 
    \begin{align*}
        \Regret_j(T) \geq \Omega\Big(\sqrt{\Qwidths{j} \log |\cA_0|}\sqrt{T}\Big),
    \end{align*}
    establishing the induction basis.
    
    \textbf{Case 2:} For any other decision point $h = j \in \cJ$ that is non-terminal, let $a_* = \argmax_{a \in \cA} \Qwidths{ja}$ be the action that maximizes the leaf count in the subtree (breaking ties arbitrarily). Let $\cI_{ja_*}(T) = \{\yb_{ja_*, t}, \rb_{ja_*, t}\}_{t=1}^T$ be the hard instance built in the subtree satisfying the induction hypothesis. We construct the hard instance $\cI_{j}(T) = \{\yb_{ja_*, t}, \rb_{ja_*, t}\}_t$ according to 
    \begin{align*}
        \yb_{j,t}[h] = \begin{cases}
            \yb_{ja_*,t}[\sigma] & \mathrm{if}\ \sigma \in \leaveseqs(ja_*) \\
            \perp & \mathrm{otherwise}
        \end{cases},
        \rb_{j,t}[\sigma] = \begin{cases}
            \rb_{ja_*,t}[\sigma] & \mathrm{if}\ \sigma \in \leaveseqs(ja_*) \\
            0 & \mathrm{otherwise}
        \end{cases}
    \end{align*}
    where $\perp$ refers to any valid transition kernel. In other words, the construction ensures taking action $a_*$ leads to the hard instance $\cI_{ja*}(T)$ while taking any other actions always results in a reward of $0$.

    Since any action different from $a_*$ will result in a reward of $0$, any algorithm $\texttt{Alg}$ that ever assigned weight outside $ja_*$ end up in a lower cumulative reward compared to the other algorithm $\texttt{Alg}'$ that consistently play action $a_*$ on decision point $j$. Together with the fact that playing $a_*$ will reduce the problem to $\cI_{ja_*}(T)$, no algorithm can achieve a higher reward on the hard instance $\cI_{j}(T)$ comparing to $\cI_{ja_*}(T)$.
    Moreover, the strategy profile given by choosing $a_*$ then following the optimal strategy profile on $\cI_{ja_*}(T)$ leads to the same cumulative reward as the optimal strategy profile on $\cI_{ja_*}(T)$. This indicates that the cumulative reward of the optimal strategy profile on $\cI_{j}(T)$ is no less than its counterpart on $\cI_{ja_*}(T)$.
    Combining the two statements, we have that the regret lower bound on $\cI_{j}(T)$ is no less than the regret lower bound on $\cI_{ja_*}(T)$. This implies the regret of any algorithm \texttt{Alg} can be lower bounded by 
    \begin{align*}
        \Regret_j(T) \geq \Omega\Big(\sqrt{\Qwidths{ja_*} \log |\cA_0|}\sqrt{T}\Big) = \Omega\Big(\sqrt{\Qwidths{j} \log |\cA_0|}\sqrt{T}\Big).
    \end{align*}
    where the last equality follows from Lemma~\ref{lm:qwidth-recusrive}, which indicates $\Qwidths{j} = \max_{a \in \cA_j} \Qwidths{ja} = \Qwidths{ja_*}$, where the last equality follows from the choice of action $a_*$.

    \textbf{Case 3:} 
    If $h = \sigma \in \nonleaveseqs$ is some non-terminal observation point, we construct the hard instance $\cI_{\sigma}(T)$ by concatenating several hard instance blocks, each block for one observation outcome $j \in \cC_\sigma$. For $j \in \cC_\sigma$, taking observation at point $\sigma$ always leads to decision point $j$. Let $\{T_j \in \ZZ_{\geq 0}\}_{j \in \cC_\sigma}$ be a partition which maximizes $\sum_{j} T_j \Qwidths{j}$ under the constraint $\sum_{j} T_j = T$. This ensures $T_j \approx T\Qwidths{j} / \Qwidths{\sigma}$. 
    We construct the hard instance block $\cI_{\sigma\rightarrow j}(T_j) = \{\rb_{\sigma,t}, \yb_{\sigma,t}\}_{t=1}^{T_j}$ using the hard instance $\cI_{j}(T_j) = \{\rb_{j,t}, \yb_{j,t}\}_{t=1}^{T_j}$ in $\tfsdp_j$ given by the induction, by assigning
    \begin{align*}
        \yb_{\sigma,t}[\sigma'] = \begin{cases}
            \yb_{j,t}[\sigma'] & \mathrm{if}\ \sigma' \in \leaveseqs(j) \\
            \perp & \mathrm{otherwise}
        \end{cases},
        \rb_{\sigma,t}[\sigma'] = \begin{cases}
            \rb_{j,t}[\sigma'] & \mathrm{if}\ \sigma' \in \leaveseqs(j) \\
            0 & \mathrm{otherwise}
        \end{cases}.
    \end{align*}
    We construct the hard instance $\cI_\sigma(T)$ by concatenating $\cI_{\sigma \rightarrow j}(T_j)$ over episodes. 
    
    From the property of observation point, the cumulative regret of any algorithm on hard instance $\cI_{\sigma \rightarrow j}(T_j)$ is equal to that on $\cI_{j}(T_j)$. Since hard instances $\cI_{j}(T_j)$ are independent on the decision space, the cumulative regret on $\cI_\sigma(T)$ is the summation of the cumulative regret among all $\cI_{\sigma \rightarrow}(T_j)$. This indicates that any algorithm \textrm{Alg} will suffer a regret of at least 
    \begin{align*}
        \Regret_\sigma(T) &= \sum_{j \in \cC_\sigma} \Regret_{\sigma \rightarrow j}(T_j) \\
        &= \sum_{j \in \cC_\sigma} \Regret_{j}(T_j) \\
        &\geq \sum_{j \in \cC_\sigma} \Omega\Big(\sqrt{ \Qwidths{j}\log|\cA_0|}\sqrt{T_j} \Big) \\
        &\geq \sum_{j \in \cC_\sigma} \Omega\Big(\sqrt{\Qwidths{j}\log |\cA_0|} \sqrt{T\Qwidths{j}/\Qwidths{\sigma}}\Big) \\ &= \Omega(\sqrt{\Qwidths{\sigma} \log |\cA_0|}\sqrt{T}),
    \end{align*}
    where the second inequality is given by the assignment of $T_j$ and the last equality is given by Lemma~\ref{lm:qwidth-recusrive} in which $\Qwidths{\sigma} = \sum_{j \in \cC_\sigma} \Qwidths{j}$.
        
    In general, the induction hypothesis always holds, indicating that for any algorithm \texttt{Alg}, it suffers an expected regret of at least $\Omega(\sqrt{\Qwidths{h} \log |\cA_0|}\sqrt{T})$ in subtree $h$. The desired result can be reached by inspecting $h = \emptyset$.
\end{proof}

\subsection{Proof of Lemma~\ref{lm:ub-width-size-conv}}

We first establish a structural result for TFSDPs where no non-root observation point yield exactly one outcome.
\begin{lemma} \label{lm:ub-width-size-conv-1}
    For a TFSDP with a given point set $\tfsdp$, if no non-root observation point yield exactly one outcome, that is, $|\cC_\sigma| \geq 2$ for any $\sigma \in \nonrootseqs \setminus \leaveseqs$, then it follows that $\Qsize \leq 2\Qwidth$.
\end{lemma}

\begin{proof}
    We prove the statement by induction on $\tfsdp$ from the bottom up, showing that for any strategy profile $\xb \in \cQ$, and any non-root point $h \in \tfsdp \setminus \{\emptyset\}$, it satisfies that 
    \begin{align*}
        \|\xb[\seqs_h]\|_1 := \sum_{\sigma \in \seqs_h} \xb[\sigma] \leq \xb[h] \cdot (2\Qwidths{h} - 1) .
    \end{align*}
    
    \textbf{Case 1:} The base case for the induction occurs when $h = \sigma \in \leaveseqs$ is a terminal node. In this scenario, the statement holds true since 
    \[\sum_{\sigma \in \seqs_h} \xb[\sigma] = \xb[h]  \leq \xb[\sigma] \cdot (2\Qwidths{\sigma} - 1) .\]
    where the last inequality is given by Lemma~\ref{lm:qwidth-recusrive} in which $2\Qwidths{\sigma} = 1$.

    \textbf{Case 2:} For any decision point $h = j \in \cJ$, it holds that 
    \begin{align*}
        \sum_{\sigma \in \seqs_j} \xb[\sigma] &= \sum_{a \in \cA_j} \sum_{\sigma \in \seqs_{ja}} \xb[\sigma] \\
        &\leq \sum_{a \in \cA_j} \xb[ja] \cdot (2\Qwidths{ja} - 1) \\
        &\leq \sum_{a \in \cA_j} \xb[ja] \cdot (2\Qwidths{j} - 1) \\
        &= \xb[j] \cdot (2\Qwidths{j} - 1),
    \end{align*}
    where the first equality follows from to the tree hierarchy of $\seqs$, the first inequality follows the induction hypothesis, the second inequality holds since $\Qwidths{ja} \leq \Qwidths{j}$ implied by Lemma~\ref{lm:qwidth-recusrive}, and the last equality follows from $\xb[j] = \sum_{a \in \cA_j} \xb[ja]$ as $\xb \in \cQ$. This indicates that $\sum_{\sigma \in \seqs(h)} \xb[\sigma] \leq (2\Qwidths{h} - 1) \cdot \xb[h]$ holds in this case.

    \textbf{Case 3:} For any non-root-non-terminal observation point $h = \sigma \in \nonrootseqs \setminus \leaveseqs$, it holds that 
    \begin{align*}
        \sum_{\sigma' \in \seqs_\sigma} \xb[\sigma'] &= \xb[\sigma] + \sum_{j \in \cC_\sigma} \sum_{\sigma' \in \seqs_{j}} \xb[\sigma'] \\
        &\leq \xb[\sigma] + \sum_{j \in \cC_\sigma} \xb[j] \cdot (2\Qwidths{j} - 1)  \\
        &= \xb[\sigma] \cdot (1 + 2\Qwidths{\sigma} - |\cC_\sigma|)  \\
        &\leq \xb[\sigma] \cdot (2\Qwidths{\sigma} - 1) ,
    \end{align*}
    where the the first equality holds due to the tree hierarchy of $\seqs$, the first inequality follows the induction hypothesis, the second equality holds since $\xb[j'] = \xb[\sigma]$ for $j' \in \cC_\sigma$ from $\xb \in \cQ$ as well as $\sum_{j' \in \cC_\sigma} \Qwidths{j'} = \Qwidths{\sigma}$ from Lemma~\ref{lm:qwidth-recusrive}, and the last inequality follows from the assumption that $|\cC_\sigma| \geq 2$.

    In general, for any non-root point $h \in \nonrootseqs$, it satisfies that 
    \begin{align*}
        \sum_{\sigma \in \seqs_h} \xb[\sigma] \leq (2\Qwidths{h} - 1) \cdot \xb[h].
    \end{align*}

    Now consider the root point $h = \emptyset$, we have that 
    \begin{align*}
        \sum_{\sigma \in \seqs} \xb[\sigma] &= \xb[\emptyset] + \sum_{j \in \cC_{\emptyset}} \sum_{\sigma \in \seqs_j} \xb[\sigma] \\
        &\leq \xb[\emptyset] + \sum_{j \in \cC_\emptyset} \xb[j] \cdot (2\Qwidths{j} - 1)  \\
        &= 1 + 2\Qwidths{\sigma} - |\cC_\sigma|  \\
        &\leq 2\Qwidths{\sigma},
    \end{align*}
    where the the first equality holds due to the tree hierarchy of $\seqs$, the first inequality follows the induction hypothesis, the second equality holds since $\xb[j'] = \xb[\emptyset] = 1$ for $j' \in \cC_\emptyset$ from $\xb \in \cQ$ as well as $\sum_{j' \in \cC_\emptyset} \Qwidths{j'} = \Qwidths{\emptyset}$ from Lemma~\ref{lm:qwidth-recusrive}, and the last inequality follows from  $|\cC_\emptyset| \geq 1$. According to the definition that $\Qwidth = \max \sum_{\sigma \in \seqs} \xb[\sigma]$, we conclude that $\Qwidth \leq 2 \Qsize$ if $|\cC_\sigma| \geq 2$ for any $\sigma \in \nonrootseqs \setminus \leaveseqs$.
\end{proof}

The next lemma shows any TFSDP can be transformed into a TFSDP where no non-root observation point yield exactly one outcome.
\begin{lemma} \label{lm:ub-width-size-conv-2}
    Given TFSDP with a given point set $\tfsdp_0$, it can always be represented using another TFSDP $\tfsdp$ with the same compressed extensive-form decision space $\cQ$ such that no non-root observation point yield exactly one outcome. Furthermore, the leaf count $\Qwidth$ remains unchanged after the transformation, while the total number of actions $\sum_{j \in \cJ} |\cA_j|$ does not increase.
\end{lemma}

\begin{proof}
    We first present a transformation which removes each non-root observation point that yields only one outcome while makes the compressed extensive-form decision space $\cQ$ remain unchanged. 

    Let $\cH_0$ be a TFSDP, where is some non-root observation point $\sigma = ja \in \nonrootseqs$ yields only one outcome, say $\cC_{\sigma} = \{j'\}$ for a single $j'$. 
    The transformation is achieved by examining the local reduced normal-form strategy at decision point $j$. As taking action $a$ at decision point $j$ invariably leads to the state transition to decision point $j'$, we can dictate the agent’s actions at $j'$ contingent on the choice of action $a$ at point $j$. This eliminates the observation point $\sigma$ while the compressed extensive-form decision space $\cQ$ remains unchanged, and the number of actions is reduced by $1$ since $ja$ is eliminated. We present an example of this transformation in Figure~\ref{fig:proper_tfsdp}.

    Since the number of actions is bounded, this process will always terminate. At this stage, there is no non-root observation point yield exactly one outcome, while the compressed extensive-form decision space $\cQ$ remains unchanged. The total number of actions $\sum_{j \in \cJ} |\cA_j|$ does not increase. Furthermore, since the leaf count $\Qwidth$ can be determined by $\cQ$ alone, this suggests the leaf count $\Qwidth$ also remains unchanged after the transformation.
    
    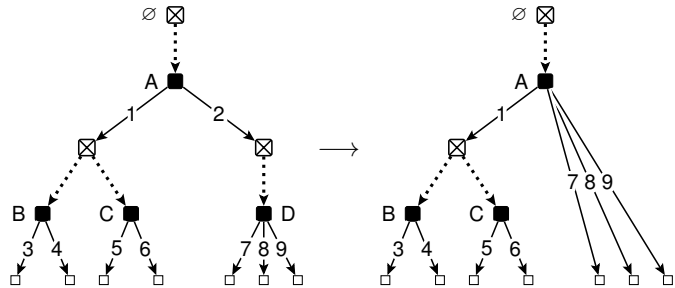
\begin{figure}[htp]\centering%
        \raisebox{1cm}{\scalebox{.98}{

\begin{tikzpicture}[baseline=0pt,scale=1]
    \def\done{1.2}
    \def\dtwo{.6}
    \def\dleaf{.37}
    \def\dvert{-0.9}
    \contourlength{.6mm}
    \def\Xseq#1{\scalebox{.8}{\contour{white}{\seq{#1}}}}

    \node[obspt] (E) at ($(0, -\dvert)$) {};
    \node[decpt] (A) at (0, 0) {};
    \node[obspt] (X) at ($(-\done,\dvert)$) {};
    \node[obspt] (Y) at ($(\done,\dvert)$) {};
    \node[decpt] (D) at ($(Y) + (0,\dvert)$) {};
    \node[decpt] (B) at ($(X) + (-\dtwo, \dvert)$) {};
    \node[decpt] (C) at ($(X) + (\dtwo, \dvert)$) {};

    \node[draw=black,inner sep=.6mm] (l1) at ($(B) + (-\dleaf, \dvert)$) {};
    \node[draw=black,inner sep=.6mm] (l2) at ($(B) + (\dleaf, \dvert)$) {};
    \node[draw=black,inner sep=.6mm] (l3) at ($(C) + (-\dleaf, \dvert)$) {};
    \node[draw=black,inner sep=.6mm] (l4) at ($(C) + (\dleaf, \dvert)$) {};
    \node[draw=black,inner sep=.6mm] (l5) at ($(D) + (-1.25*\dleaf, \dvert)$) {};
    \node[draw=black,inner sep=.6mm] (l6) at ($(D) + (0, \dvert)$) {};
    \node[draw=black,inner sep=.6mm] (l7) at ($(D) + (1.25*\dleaf, \dvert)$) {};

    \draw[action] (A) -- node{\Xseq{1}} (X);
    \draw[action] (A) -- node{\Xseq{2}} (Y);
    \draw[action] (B) -- node{\Xseq{3}} (l1);
    \draw[action] (B) -- node{\Xseq{4}} (l2);
    \draw[action] (C) -- node{\Xseq{5}} (l3);
    \draw[action] (C) -- node{\Xseq{6}} (l4);
    \draw[action] (D) -- node{\Xseq{7}} (l5);
    \draw[action] (D) -- node[fill=white,inner sep=0mm]{\Xseq{8}} (l6);
    \draw[action] (D) -- node{\Xseq{9}} (l7);
    \draw[observ] (E) -- (A);
    \draw[observ] (X) -- (B);
    \draw[observ] (X) -- (C);
    \draw[observ] (Y) -- (D);


    \node[black, left=0mm of E] {\scalebox{.8}{\decpt{$\emptyset$}}};
    \node[black, left=0mm of A] {\scalebox{.8}{\decpt{a}}};
    \node[black, left=0mm of B] {\scalebox{.8}{\decpt{b}}};
    \node[black, left=0mm of C] {\scalebox{.8}{\decpt{c}}};
    \node[black,right=0mm of D] {\scalebox{.8}{\decpt{d}}};
\end{tikzpicture}

        }}%
        ~$\longrightarrow$~%
        \raisebox{1cm}{\scalebox{.98}{

        \begin{tikzpicture}[baseline=0pt,scale=1]
    \def\done{1.2}
    \def\dtwo{.6}
    \def\dleaf{.37}
    \def\dvert{-0.9}
    \contourlength{.6mm}
    \def\Xseq#1{\scalebox{.8}{\contour{white}{\seq{#1}}}}

    \node[obspt] (E) at ($(0, -\dvert)$) {};
    \node[decpt] (A) at (0, 0) {};
    \node[obspt] (X) at ($(-\done,\dvert)$) {};
    \node (Y) at ($(\done,\dvert)$) {};
    \node (D) at ($(Y) + (0,\dvert)$) {};
    \node[decpt] (B) at ($(X) + (-\dtwo, \dvert)$) {};
    \node[decpt] (C) at ($(X) + (\dtwo, \dvert)$) {};

    \node[draw=black,inner sep=.6mm] (l1) at ($(B) + (-\dleaf, \dvert)$) {};
    \node[draw=black,inner sep=.6mm] (l2) at ($(B) + (\dleaf, \dvert)$) {};
    \node[draw=black,inner sep=.6mm] (l3) at ($(C) + (-\dleaf, \dvert)$) {};
    \node[draw=black,inner sep=.6mm] (l4) at ($(C) + (\dleaf, \dvert)$) {};
    \node[draw=black,inner sep=.6mm] (l5) at ($(D) + (-1.25*\dleaf, \dvert)$) {};
    \node[draw=black,inner sep=.6mm] (l6) at ($(D) + (0, \dvert)$) {};
    \node[draw=black,inner sep=.6mm] (l7) at ($(D) + (1.25*\dleaf, \dvert)$) {};

    \draw[action] (A) -- node{\Xseq{1}} (X);
    \draw[action] (B) -- node{\Xseq{3}} (l1);
    \draw[action] (B) -- node{\Xseq{4}} (l2);
    \draw[action] (C) -- node{\Xseq{5}} (l3);
    \draw[action] (C) -- node{\Xseq{6}} (l4);
    \draw[action] (A) -- node{\Xseq{8}} (l6);
    \draw[action] (A) -- node{\Xseq{9}} (l7);
    \draw[action] (A) -- node{\Xseq{7}} (l5);
    \draw[observ] (E) -- (A);
    \draw[observ] (X) -- (B);
    \draw[observ] (X) -- (C);


    \node[black, left=0mm of E] {\scalebox{.8}{\decpt{$\emptyset$}}};
    \node[black, left=0mm of A] {\scalebox{.8}{\decpt{a}}};
    \node[black, left=0mm of B] {\scalebox{.8}{\decpt{b}}};
    \node[black, left=0mm of C] {\scalebox{.8}{\decpt{c}}};
\end{tikzpicture}

        }}%

        \caption{Eliminating observation point $\textsf{A2}$ from the TFSDP. The compressed extensive-form decision space $\cQ$ remains unchanged and still has support $\{\textsf{3}, \textsf{4}, \textsf{5}, \textsf{6}, \textsf{7}, \textsf{8}, \textsf{9}\}$. Thus, the leaf count for the new TFSDP remains $\Qwidth = 2$.    
        Furthermore the total number of actions is reduced by one. }

        \label{fig:proper_tfsdp}
    \end{figure}
\end{proof}

\restateLemma{lm:ub-width-size-conv}
\begin{proof}[Proof of Lemma~\ref{lm:ub-width-size-conv}]
    According to Proposition 5.1 in \citet{farina2022kernelized}, we have $|\cV| \leq |\cA|^{\Qsize}$. When there is no non-root observation point yields exactly one outcome, we have 
    $\Qsize \leq 2\Qwidth$ from Lemma~\ref{lm:ub-width-size-conv-1}. This implies $\ln |\cV| \leq \cO(\Qsize \log |\cA|)$.

    In the general cases, we can use Lemma~\ref{lm:ub-width-size-conv-2} to transform the TFSDP into the desired representation. In this case, neither the leaf count $\Qwidth$ nor the total number of actions $\sum_{j \in \cJ} |\cA_j|$ increases. The size of the largest action set can be upper bounded using the total number of actions, which can be further upper bounded from $\sum_{j \in \cJ} |\cA_j| \leq |\cJ \times \cA|$. This implies $\ln |\cV| \leq \cO(\Qsize \log |\cJ \times \cA|)$ always holds.
    
\end{proof}

\section{Auxiliary Lemmas}

\begin{lemma} \label{lm:qwidth-recusrive}
    The leaf count $\Qwidth = \Qwidths{\emptyset}$ can be computed recursively as:
    \begin{itemize}[nosep]
        \item  If $h = \sigma \in \leaveseqs$ is a terminal observation point, then:
        \[\textstyle{
            \Qwidths{\sigma} := 1.
        }\]
        \item  If $h = j \in \cJ$ is a decision point, then:
        \[\textstyle{
            \Qwidths{j} := \max_{a \in \cA_j} \Qwidths{ja}.
        }\]
        \item  If $h = \sigma \in \nonleaveseqs$ is a non-terminal observation point, then: 
        \[\textstyle{
            \Qwidths{\sigma} := \sum_{j \in \cC_\sigma} \Qwidths{j}.
        }\]
    \end{itemize}
\end{lemma}

\begin{proof}
    According to the definition of leaf count and the \tpLInfNorm, we have 
    \begin{align*}
        \Qwidth = \sup_{\xb \in \cQ} \|\xb\|_1 = \sup_{\xb \in \cQ} \langle|\one|, \xb\rangle = \|\one\|\Qnorm.
    \end{align*}
    This establishes a connection between leaf count with the \tpLInfNorm. According to Lemma~\ref{lm:tpnorm-recursive}, we immediately reach the desired statement.
\end{proof}

\begin{lemma} \label{lm:pure-startegy-count-recursive}
    The number of pure strategy profiles $|\cV| = |\cV_\emptyset|$ can be computed recursively as 
    \begin{itemize}[nosep]
        \item  If $h = \sigma \in \leaveseqs$ is a terminal observation point, then:
        \[\textstyle{
            |\cV_\sigma| := 1.
        }\]
        \item  If $h = j \in \cJ$ is a decision point, then:
        \[\textstyle{
            |\cV_j| := \sum_{a \in \cA_j} |\cV_{ja}|.
        }\]
        \item  If $h = \sigma \in \nonleaveseqs$ is a non-terminal observation point, then: 
        \[\textstyle{
            |\cV_\sigma| := \prod_{j \in \cC_\sigma} |\cV_j|.
        }\]
    \end{itemize}
\end{lemma}

\begin{proof}
    According to the definition of reduced normal-form strategies, the statement can be immediately reached by inspecting the vertices of each extensive-form strategy space $\cQ_h$. 
\end{proof}

\clearpage
\input{checklist.tex}

\end{document}

%% file: macros.tex
\newcommand{\Regret}{\textrm{Regret}}

\newcommand{\DilEnt}{{\varphi_\one}}
\newcommand{\DilEntw}[1]{{\varphi_{#1}}}
\newcommand{\DilEnts}[1]{{\varphi_{\one,#1}}}
\newcommand{\seqs}{\Sigma}
\newcommand{\nonrootseqs}{{\Sigma^+}}
\newcommand{\tfsdp}{\cH}
\newcommand{\leaveseqs}{\cE}
\newcommand{\nonleaveseqs}{{\Sigma \setminus \leaveseqs}}

\newcommand{\Qnorm}{{}_{\tfsdp,\infty}}
\newcommand{\Qnorms}[1]{{}_{\tfsdp_{#1},\infty}}

\newcommand{\Qnormp}[1]{{}_{\tfsdp^{#1},\infty}}

\newcommand{\Ynorm}{{}_{\tfsdp,1}}
\newcommand{\Ynorms}[1]{{}_{\tfsdp_{#1},1}}

\newcommand{\Ynormp}[1]{{}_{\tfsdp^{#1},1}}

\newcommand{\Qwidth}{\|\cQ\|_{\perp}}
\newcommand{\Qwidths}[1]{\|\cQ_{#1}\|_{\perp}}
\newcommand{\Qsize}{\|\cQ\|_1}

\newcommand{\efg}{\cT}
\newcommand{\efgleaves}{\cZ}

\newcommand{\tpLOneNorm}{treeplex $\ell_1$ norm}
\newcommand{\tpLInfNorm}{treeplex $\ell_\infty$ norm}

\usepackage{environ}
\makeatletter
\NewEnviron{Lemma}[1]{%
  \protected@write\@auxout{}{%
    \string\@restatelemma{#1}{\detokenize\expandafter{\BODY}}%
  }%
  \begin{lemma}\label{#1}\BODY\end{lemma}%
}
\newcommand{\@restatelemma}[2]{%
  \expandafter\gdef\csname Lemma@#1\endcsname{#2}%
}
\newcommand{\restateLemma}[1]{%
  \begingroup
  \renewcommand{\thetheorem}{\ref{#1}}%
  \begin{lemma}[restatement]\csname Lemma@#1\endcsname\end{lemma}%
  \addtocounter{theorem}{-1}
  \endgroup
}
\makeatother

\makeatletter
\NewEnviron{Theorem}[1]{%
  \protected@write\@auxout{}{%
    \string\@restatetheorem{#1}{\detokenize\expandafter{\BODY}}%
  }%
  \begin{theorem}\label{#1}\BODY\end{theorem}%
}
\newcommand{\@restatetheorem}[2]{%
  \expandafter\gdef\csname Theorem@#1\endcsname{#2}%
}
\newcommand{\restateTheorem}[1]{%
  \begingroup
  \renewcommand{\thetheorem}{\ref{#1}}%
  \begin{theorem}[restatement]\csname Theorem@#1\endcsname\end{theorem}%
  \addtocounter{theorem}{-1}
  \endgroup
}
\makeatother

\usepackage[colorlinks,citecolor=blue]{hyperref}
\usepackage{xcolor}
\usepackage{colortbl}
\definecolor{LightCyan}{rgb}{0.8, 0.9, 1}
\definecolor{LightGray}{rgb}{0.9,0.9,0.9}

\usepackage{stmaryrd}
\usepackage{enumitem}
\usepackage[capitalise]{cleveref}

\usepackage{algorithm2e}
\RestyleAlgo{ruled}

\usepackage{nicematrix}

\let\emptyset\varnothing
\DeclareMathOperator*{\Expt}{\mathbb{E}}

\DeclareMathOperator*{\Unif}{\mathrm{Unif}}


%% file: gabri_figure.tex
\usetikzlibrary{
    arrows,
    arrows.meta,
    backgrounds,
    calc,
    calligraphy,
    decorations.pathreplacing,
    fit,
    matrix,
    patterns,
    positioning,
    shapes.misc,
    shapes.geometric
}
\usepackage{pgfcalendar}
\usepackage[object=vectorian]{pgfornament}
\contourlength{.6mm} 

\NewDocumentCommand  \decpt          {m} {\textsf{\MakeUppercase{#1}}}
\NewDocumentCommand  \seq            {m} {\textsf{#1}}

\pgfdeclarelayer{background}
\pgfdeclarelayer{middle}
\pgfsetlayers{background,middle,main}

\makeatletter
\tikzset{
    fitting node/.style={
            inner sep=0pt,
            fill=none,
            draw=none,
            reset transform,
            fit={(\pgf@pathminx,\pgf@pathminy) (\pgf@pathmaxx,\pgf@pathmaxy)}
        },
    reset transform/.code={\pgftransformreset}
}

\tikzset{cross/.style={path picture={
                    \draw[black] (path picture bounding box.south east) --
                    (path picture bounding box.north west) %
                    (path picture bounding box.south west) --
                    (path picture bounding box.north east);
                }}}

\tikzset{slashed/.style={path picture={
                    \draw[black] (path picture bounding box.south west) --
                    (path picture bounding box.north east);
                }}}

\tikzset{plus/.style={path picture={
                    \draw[black,line cap=round]
                    ($(path picture bounding box.west)!.2!(path picture bounding box.east)$)
                    -- ($(path picture bounding box.west)!.8!(path picture bounding box.east)$)
                    ($(path picture bounding box.south)!.2!(path picture bounding box.north)$)
                    -- ($(path picture bounding box.south)!.8!(path picture bounding box.north)$);
                }}}

\tikzset{hatch/.style={path picture={
                    \draw[black,thin,line cap=round]
                    (path picture bounding box.south west)
                    ++(0,0.00) -- +(1,1)
                    ++(0,0.075) -- +(1,1)
                    ++(0,0.075) -- +(1,1)
                    (path picture bounding box.south west)
                    ++(0,-0.075) -- +(1,1)
                    ++(0,-0.075) -- +(1,1)
                    ;
                }}}

\tikzstyle{dot}=[thin,draw,fill,circle,inner sep=.5mm]

\tikzstyle{lbl} = [fill=white,rounded corners,inner ysep=.7mm]
\tikzstyle{tightlbl} = [lbl,inner xsep=.6mm,inner ysep=.2mm]

\tikzstyle{proc}=[thick,draw=black,fill=white,slashed,inner sep=2mm,rounded corners=.5mm]
\tikzstyle{oplus}=[thick,draw=black,fill=white,circle,plus,inner sep=1.4mm]

\tikzstyle{@hlg}=[draw,tips=false,line width=1.7mm,white,rounded corners=.5mm]
\tikzstyle{uarr}=[preaction={@hlg},thick,->,black!20!red ,dashed,rounded corners=.5mm]
\tikzstyle{xarr}=[preaction={@hlg},thick,->,black!20!blue,       rounded corners=.5mm]

\tikzstyle{rm} = [draw=black,thick,minimum height=8mm,minimum width=12mm,fill=white,rounded corners=.5mm]
\colorlet{hlbg}{black!3}
\tikzstyle{hlbox} = [ultra thick,draw=black!12,fill=hlbg,rounded corners=.8mm]

\tikzstyle{chance} = [fill=white, draw=black, circle, semithick,cross,inner sep=.8mm]
\tikzstyle{pl1}    = [fill=black, draw=black, circle, inner sep=.5mm]
\tikzstyle{pl2}    = [fill=white, draw=black, circle, inner sep=.55mm]
\tikzstyle{termina} = [fill=white, draw=black,  inner sep=.7mm]
\tikzstyle  {decpt} = [fill=black, draw=black,         inner sep=1.0mm,rounded corners=.3mm]

\tikzstyle {action} = [
    semithick, black, -{Stealth[scale width=1.10,inset=.3mm,length=1.55mm]},
    preaction={
        draw,line width=1.0mm,white,-{Stealth[scale width=1.10,inset=.3mm,length=1.55mm]}
    }
]

\tikzstyle  {obspt} = [fill=white, draw=black, cross, semithick,  inner sep=1.1mm,rounded corners=.3mm]
\tikzstyle {observ} = [very thick, black, dotted, -{Stealth[scale width=1.10,inset=.3mm,length=1.55mm]}]
\tikzstyle     {ox} = [semithick,  draw=black, circle, cross, inner sep=1.2mm]
\tikzstyle{infoset} = [black!50!white]

\tikzstyle     {efce}=[fill=red!15!white,draw=red,very thick]
\tikzstyle{majortick}=[black,semithick]
\tikzstyle{majorgrid}=[very thin,gray!50]
\tikzstyle{minorgrid}=[gray!20,dashed,very thin]

\makeatother

\NewDocumentCommand \SingletonInfoset {O{4mm}m} {
    \Infoset[#1]{(#2.center) -- +(0.0001,0)}
}

\NewDocumentCommand \Infoset {O{4mm}m} {
    \begin{pgfonlayer}{background}
        \draw[black!40, line width/.evaluated={#1}, line cap=round] #2;
        \draw[white, line width/.evaluated={.9*#1}, line cap=round] #2;
    \end{pgfonlayer}
    \begin{pgfonlayer}{middle}
        \draw[black!10, draw opacity=1.00, line width/.evaluated={.9*#1}, line cap=round] #2;
    \end{pgfonlayer}
}

\NewDocumentCommand \hlSubmatrix {O{black!15}mm} {%
    \filldraw[draw=black!40,fill=#1,rounded corners=1mm] (#2.north west) rectangle (#3.south east)
}

%% file: checklist.tex

\newpage
\section*{NeurIPS Paper Checklist}

\begin{enumerate}

\item {\bf Claims}
    \item[] Question: Do the main claims made in the abstract and introduction accurately reflect the paper's contributions and scope?
    \item[] Answer: \answerYes{} 
    \item[] Justification: We study the optimality of weight-one dilated entropy from a theoretical perceptive. The contribution, assumptions and scope are clearly claimed in the abstract and introduction.

\item {\bf Limitations}
    \item[] Question: Does the paper discuss the limitations of the work performed by the authors?
    \item[] Answer: \answerYes{} 
    \item[] Justification: We have discussed the limitation in the conclusion section.

\item {\bf Theory Assumptions and Proofs}
    \item[] Question: For each theoretical result, does the paper provide the full set of assumptions and a complete (and correct) proof?
    \item[] Answer: \answerYes{} 
    \item[] Justification: We provide detailed proof for all theorems in the Appendix~\ref{sec:proof-to-tp-norm}, \ref{sec:proof-to-ub}, and \ref{sec:proof-to-lb}.

    \item {\bf Experimental Result Reproducibility}
    \item[] Question: Does the paper fully disclose all the information needed to reproduce the main experimental results of the paper to the extent that it affects the main claims and/or conclusions of the paper (regardless of whether the code and data are provided or not)?
    \item[] Answer: \answerNA{} 
    \item[] Justification: The paper focuses on theoretical understanding and does not include experiments.

\item {\bf Open access to data and code}
    \item[] Question: Does the paper provide open access to the data and code, with sufficient instructions to faithfully reproduce the main experimental results, as described in supplemental material?
    \item[] Answer: \answerNA{} 
    \item[] Justification: The paper focuses on theoretical understanding and does not include experiments.

\item {\bf Experimental Setting/Details}
    \item[] Question: Does the paper specify all the training and test details (e.g., data splits, hyperparameters, how they were chosen, type of optimizer, etc.) necessary to understand the results?
    \item[] Answer: \answerNA{} 
    \item[] Justification: The paper focuses on theoretical understanding and does not include experiments.

\item {\bf Experiment Statistical Significance}
    \item[] Question: Does the paper report error bars suitably and correctly defined or other appropriate information about the statistical significance of the experiments?
    \item[] Answer: \answerNA{}{} 
    \item[] Justification: The paper focuses on theoretical understanding and does not include experiments.

\item {\bf Experiments Compute Resources}
    \item[] Question: For each experiment, does the paper provide sufficient information on the computer resources (type of compute workers, memory, time of execution) needed to reproduce the experiments?
    \item[] Answer: \answerNA{} 
    \item[] Justification: The paper focuses on theoretical understanding and does not include experiments.
    
\item {\bf Code Of Ethics}
    \item[] Question: Does the research conducted in the paper conform, in every respect, with the NeurIPS Code of Ethics \url{https://neurips.cc/public/EthicsGuidelines}?
    \item[] Answer: \answerYes{} 
    \item[] Justification: We have reviewed the NeurIPS Code of Ethics.

\item {\bf Broader Impacts}
    \item[] Question: Does the paper discuss both potential positive societal impacts and negative societal impacts of the work performed?
    \item[] Answer: \answerNA{} 
    \item[] Justification: Our work serves as foundational research for algorithmic game theory. Although there might be some potential social impacts on solving large-scale games, according to the guidelines, we believe our result does not have a direct connection with these issues.
    
\item {\bf Safeguards}
    \item[] Question: Does the paper describe safeguards that have been put in place for responsible release of data or models that have a high risk for misuse (e.g., pretrained language models, image generators, or scraped datasets)?
    \item[] Answer: \answerNA{} 
    \item[] Justification: No such risks. 

\item {\bf Licenses for existing assets}
    \item[] Question: Are the creators or original owners of assets (e.g., code, data, models), used in the paper, properly credited and are the license and terms of use explicitly mentioned and properly respected?
    \item[] Answer: \answerNA{} 
    \item[] Justification: Does not apply. 

\item {\bf New Assets}
    \item[] Question: Are new assets introduced in the paper well documented and is the documentation provided alongside the assets?
    \item[] Answer: \answerNA{} 
    \item[] Justification: Does not apply. 

\item {\bf Crowdsourcing and Research with Human Subjects}
    \item[] Question: For crowdsourcing experiments and research with human subjects, does the paper include the full text of instructions given to participants and screenshots, if applicable, as well as details about compensation (if any)? 
    \item[] Answer: \answerNA{} 
    \item[] Justification: No human subjects nor crowdsourcing. 

\item {\bf Institutional Review Board (IRB) Approvals or Equivalent for Research with Human Subjects}
    \item[] Question: Does the paper describe potential risks incurred by study participants, whether such risks were disclosed to the subjects, and whether Institutional Review Board (IRB) approvals (or an equivalent approval/review based on the requirements of your country or institution) were obtained?
    \item[] Answer: \answerNA{} 
    \item[] Justification: No human subjects nor crowdsourcing. 

\end{enumerate}

%% file: neurips_2024.bbl
\begin{thebibliography}{43}
\providecommand{\natexlab}[1]{#1}
\providecommand{\url}[1]{\texttt{#1}}
\expandafter\ifx\csname urlstyle\endcsname\relax
  \providecommand{\doi}[1]{doi: #1}\else
  \providecommand{\doi}{doi: \begingroup \urlstyle{rm}\Url}\fi

\bibitem[Anagnostides et~al.(2023)Anagnostides, Kalavasis, Sandholm, and Zampetakis]{anagnostides2023complexity}
Ioannis Anagnostides, Alkis Kalavasis, Tuomas Sandholm, and Manolis Zampetakis.
\newblock On the complexity of computing sparse equilibria and lower bounds for no-regret learning in games.
\newblock \emph{arXiv preprint arXiv:2311.14869}, 2023.

\bibitem[Bai et~al.(2022{\natexlab{a}})Bai, Jin, Mei, Song, and Yu]{bai2022efficient}
Yu~Bai, Chi Jin, Song Mei, Ziang Song, and Tiancheng Yu.
\newblock Efficient {$\Phi$}-regret minimization in extensive-form games via online mirror descent.
\newblock \emph{Advances in Neural Information Processing Systems}, 35:\penalty0 22313--22325, 2022{\natexlab{a}}.

\bibitem[Bai et~al.(2022{\natexlab{b}})Bai, Jin, Mei, and Yu]{bai2022near}
Yu~Bai, Chi Jin, Song Mei, and Tiancheng Yu.
\newblock Near-optimal learning of extensive-form games with imperfect information.
\newblock In \emph{International Conference on Machine Learning}, pages 1337--1382. PMLR, 2022{\natexlab{b}}.

\bibitem[Beck and Teboulle(2003)]{beck2003mirror}
Amir Beck and Marc Teboulle.
\newblock Mirror descent and nonlinear projected subgradient methods for convex optimization.
\newblock \emph{Operations Research Letters}, 31\penalty0 (3):\penalty0 167--175, 2003.

\bibitem[Bowling et~al.(2015)Bowling, Burch, Johanson, and Tammelin]{bowling2015heads}
Michael Bowling, Neil Burch, Michael Johanson, and Oskari Tammelin.
\newblock Heads-up limit hold’em poker is solved.
\newblock \emph{Science}, 347\penalty0 (6218):\penalty0 145--149, 2015.

\bibitem[Brown and Sandholm(2018)]{brown2018superhuman}
Noam Brown and Tuomas Sandholm.
\newblock Superhuman ai for heads-up no-limit poker: Libratus beats top professionals.
\newblock \emph{Science}, 359\penalty0 (6374):\penalty0 418--424, 2018.

\bibitem[Brown and Sandholm(2019{\natexlab{a}})]{brown2019solving}
Noam Brown and Tuomas Sandholm.
\newblock Solving imperfect-information games via discounted regret minimization.
\newblock In \emph{Proceedings of the AAAI Conference on Artificial Intelligence}, volume~33, pages 1829--1836, 2019{\natexlab{a}}.

\bibitem[Brown and Sandholm(2019{\natexlab{b}})]{brown2019superhuman}
Noam Brown and Tuomas Sandholm.
\newblock Superhuman ai for multiplayer poker.
\newblock \emph{Science}, 365\penalty0 (6456):\penalty0 885--890, 2019{\natexlab{b}}.

\bibitem[Cesa-Bianchi and Lugosi(2006)]{cesa2006prediction}
Nicolo Cesa-Bianchi and G{\'a}bor Lugosi.
\newblock \emph{Prediction, learning, and games}.
\newblock Cambridge university press, 2006.

\bibitem[Chen and Teboulle(1993)]{chen1993convergence}
Gong Chen and Marc Teboulle.
\newblock Convergence analysis of a proximal-like minimization algorithm using bregman functions.
\newblock \emph{SIAM Journal on Optimization}, 3\penalty0 (3):\penalty0 538--543, 1993.

\bibitem[Chen and Peng(2020)]{chen2020hedging}
Xi~Chen and Binghui Peng.
\newblock Hedging in games: Faster convergence of external and swap regrets.
\newblock \emph{Advances in Neural Information Processing Systems}, 33:\penalty0 18990--18999, 2020.

\bibitem[Chiang et~al.(2012)Chiang, Yang, Lee, Mahdavi, Lu, Jin, and Zhu]{chiang2012online}
Chao-Kai Chiang, Tianbao Yang, Chia-Jung Lee, Mehrdad Mahdavi, Chi-Jen Lu, Rong Jin, and Shenghuo Zhu.
\newblock Online optimization with gradual variations.
\newblock In \emph{Conference on Learning Theory}, pages 6--1. JMLR Workshop and Conference Proceedings, 2012.

\bibitem[Daskalakis et~al.(2021)Daskalakis, Fishelson, and Golowich]{daskalakis2021near}
Constantinos Daskalakis, Maxwell Fishelson, and Noah Golowich.
\newblock Near-optimal no-regret learning in general games.
\newblock \emph{Advances in Neural Information Processing Systems}, 34:\penalty0 27604--27616, 2021.

\bibitem[Farina et~al.(2019{\natexlab{a}})Farina, Kroer, Brown, and Sandholm]{farina2019stable}
Gabriele Farina, Christian Kroer, Noam Brown, and Tuomas Sandholm.
\newblock Stable-predictive optimistic counterfactual regret minimization.
\newblock In \emph{International conference on machine learning}, pages 1853--1862. PMLR, 2019{\natexlab{a}}.

\bibitem[Farina et~al.(2019{\natexlab{b}})Farina, Kroer, and Sandholm]{farina2019online}
Gabriele Farina, Christian Kroer, and Tuomas Sandholm.
\newblock Online convex optimization for sequential decision processes and extensive-form games.
\newblock In \emph{Proceedings of the AAAI Conference on Artificial Intelligence}, volume~33, pages 1917--1925, 2019{\natexlab{b}}.

\bibitem[Farina et~al.(2021{\natexlab{a}})Farina, Kroer, and Sandholm]{farina2021better}
Gabriele Farina, Christian Kroer, and Tuomas Sandholm.
\newblock Better regularization for sequential decision spaces: Fast convergence rates for nash, correlated, and team equilibria.
\newblock In \emph{Proceedings of the 2021 ACM Conference on Economics and Computation}, 2021{\natexlab{a}}.

\bibitem[Farina et~al.(2021{\natexlab{b}})Farina, Kroer, and Sandholm]{farina2021faster}
Gabriele Farina, Christian Kroer, and Tuomas Sandholm.
\newblock Faster game solving via predictive blackwell approachability: Connecting regret matching and mirror descent.
\newblock In \emph{Proceedings of the AAAI Conference on Artificial Intelligence}, volume~35, pages 5363--5371, 2021{\natexlab{b}}.

\bibitem[Farina et~al.(2022{\natexlab{a}})Farina, Anagnostides, Luo, Lee, Kroer, and Sandholm]{farina2022near}
Gabriele Farina, Ioannis Anagnostides, Haipeng Luo, Chung-Wei Lee, Christian Kroer, and Tuomas Sandholm.
\newblock Near-optimal no-regret learning dynamics for general convex games.
\newblock \emph{Advances in Neural Information Processing Systems}, 35:\penalty0 39076--39089, 2022{\natexlab{a}}.

\bibitem[Farina et~al.(2022{\natexlab{b}})Farina, Kroer, Lee, and Luo]{farina2022clairvoyant}
Gabriele Farina, Christian Kroer, Chung-Wei Lee, and Haipeng Luo.
\newblock Clairvoyant regret minimization: Equivalence with {N}emirovski's conceptual prox method and extension to general convex games.
\newblock \emph{arXiv preprint arXiv:2208.14891}, 2022{\natexlab{b}}.

\bibitem[Farina et~al.(2022{\natexlab{c}})Farina, Lee, Luo, and Kroer]{farina2022kernelized}
Gabriele Farina, Chung-Wei Lee, Haipeng Luo, and Christian Kroer.
\newblock Kernelized multiplicative weights for 0/1-polyhedral games: Bridging the gap between learning in extensive-form and normal-form games.
\newblock In \emph{International Conference on Machine Learning}, pages 6337--6357. PMLR, 2022{\natexlab{c}}.

\bibitem[Fiegel et~al.(2023)Fiegel, M{\'e}nard, Kozuno, Munos, Perchet, and Valko]{fiegel2023adapting}
C{\^o}me Fiegel, Pierre M{\'e}nard, Tadashi Kozuno, R{\'e}mi Munos, Vianney Perchet, and Michal Valko.
\newblock Adapting to game trees in zero-sum imperfect information games.
\newblock In \emph{International Conference on Machine Learning}, pages 10093--10135. PMLR, 2023.

\bibitem[Hoda et~al.(2010)Hoda, Gilpin, Peña, and Sandholm]{hoda2010smoothing}
Samid Hoda, Andrew Gilpin, Javier Peña, and Tuomas Sandholm.
\newblock Smoothing techniques for computing nash equilibria of sequential games.
\newblock \emph{Mathematics of Operations Research}, 35\penalty0 (2):\penalty0 494--512, 2010.

\bibitem[Huang and von Stengel(2008)]{huang2008computing}
Wan Huang and Bernhard von Stengel.
\newblock Computing an extensive-form correlated equilibrium in polynomial time.
\newblock In \emph{International Workshop on Internet and Network Economics}, pages 506--513. Springer, 2008.

\bibitem[Jiang and Leyton-Brown(2011)]{jiang2011polynomial}
Albert~Xin Jiang and Kevin Leyton-Brown.
\newblock Polynomial-time computation of exact correlated equilibrium in compact games.
\newblock In \emph{Proceedings of the 12th ACM conference on Electronic commerce}, pages 119--126, 2011.

\bibitem[Koolen et~al.(2010)Koolen, Warmuth, and Kivinen]{koolen2010hedging}
Wouter~M Koolen, Manfred~K Warmuth, and Jyrki Kivinen.
\newblock Hedging structured concepts.
\newblock In \emph{COLT 2010: Proceedings of the 23rd Annual Conference on Learning Theory}, pages 93--105, 2010.

\bibitem[Kroer et~al.(2018)Kroer, Farina, and Sandholm]{kroer2018solving}
Christian Kroer, Gabriele Farina, and Tuomas Sandholm.
\newblock Solving large sequential games with the excessive gap technique.
\newblock \emph{Advances in neural information processing systems}, 31, 2018.

\bibitem[Kroer et~al.(2020)Kroer, Waugh, K{\i}l{\i}n{\c{c}}-Karzan, and Sandholm]{kroer2020faster}
Christian Kroer, Kevin Waugh, Fatma K{\i}l{\i}n{\c{c}}-Karzan, and Tuomas Sandholm.
\newblock Faster algorithms for extensive-form game solving via improved smoothing functions.
\newblock \emph{Mathematical Programming}, pages 1--33, 2020.

\bibitem[Lee et~al.(2021)Lee, Kroer, and Luo]{lee2021last}
Chung-Wei Lee, Christian Kroer, and Haipeng Luo.
\newblock Last-iterate convergence in extensive-form games.
\newblock \emph{Advances in Neural Information Processing Systems}, 34:\penalty0 14293--14305, 2021.

\bibitem[Liu et~al.()Liu, Ozdaglar, Yu, and Zhang]{liu2022power}
Mingyang Liu, Asuman~E Ozdaglar, Tiancheng Yu, and Kaiqing Zhang.
\newblock The power of regularization in solving extensive-form games.
\newblock In \emph{The Eleventh International Conference on Learning Representations}.

\bibitem[Liu et~al.(2024)Liu, Farina, and Ozdaglar]{liu2024policy}
Mingyang Liu, Gabriele Farina, and Asuman Ozdaglar.
\newblock A policy-gradient approach to solving imperfect-information games with iterate convergence.
\newblock \emph{arXiv preprint arXiv:2408.00751}, 2024.

\bibitem[Morav{\v{c}}{\'\i}k et~al.(2017)Morav{\v{c}}{\'\i}k, Schmid, Burch, Lis{\`y}, Morrill, Bard, Davis, Waugh, Johanson, and Bowling]{moravvcik2017deepstack}
Matej Morav{\v{c}}{\'\i}k, Martin Schmid, Neil Burch, Viliam Lis{\`y}, Dustin Morrill, Nolan Bard, Trevor Davis, Kevin Waugh, Michael Johanson, and Michael Bowling.
\newblock Deepstack: Expert-level artificial intelligence in heads-up no-limit poker.
\newblock \emph{Science}, 356\penalty0 (6337):\penalty0 508--513, 2017.

\bibitem[Nash~Jr(1950)]{nash1950equilibrium}
John~F Nash~Jr.
\newblock Equilibrium points in n-person games.
\newblock \emph{Proceedings of the national academy of sciences}, 36\penalty0 (1):\penalty0 48--49, 1950.

\bibitem[Nemirovski(2004)]{nemirovski2004prox}
Arkadi Nemirovski.
\newblock Prox-method with rate of convergence o (1/t) for variational inequalities with lipschitz continuous monotone operators and smooth convex-concave saddle point problems.
\newblock \emph{SIAM Journal on Optimization}, 15\penalty0 (1):\penalty0 229--251, 2004.

\bibitem[Nesterov(2005)]{nesterov2005excessive}
Yu~Nesterov.
\newblock Excessive gap technique in nonsmooth convex minimization.
\newblock \emph{SIAM Journal on Optimization}, 16\penalty0 (1):\penalty0 235--249, 2005.

\bibitem[Papadimitriou and Roughgarden(2008)]{papadimitriou2008computing}
Christos~H Papadimitriou and Tim Roughgarden.
\newblock Computing correlated equilibria in multi-player games.
\newblock \emph{Journal of the ACM (JACM)}, 55\penalty0 (3):\penalty0 1--29, 2008.

\bibitem[Piliouras et~al.(2022)Piliouras, Sim, and Skoulakis]{piliouras2022beyond}
Georgios Piliouras, Ryann Sim, and Stratis Skoulakis.
\newblock Beyond time-average convergence: Near-optimal uncoupled online learning via clairvoyant multiplicative weights update.
\newblock \emph{Advances in Neural Information Processing Systems}, 35:\penalty0 22258--22269, 2022.

\bibitem[Rakhlin and Sridharan(2013)]{rakhlin2013online}
Alexander Rakhlin and Karthik Sridharan.
\newblock Online learning with predictable sequences.
\newblock In \emph{Conference on Learning Theory}, pages 993--1019. PMLR, 2013.

\bibitem[Sandholm(2010)]{sandholm2010state}
Tuomas Sandholm.
\newblock The state of solving large incomplete-information games, and application to poker.
\newblock \emph{Ai Magazine}, 31\penalty0 (4):\penalty0 13--32, 2010.

\bibitem[Syrgkanis et~al.(2015)Syrgkanis, Agarwal, Luo, and Schapire]{syrgkanis2015fast}
Vasilis Syrgkanis, Alekh Agarwal, Haipeng Luo, and Robert~E Schapire.
\newblock Fast convergence of regularized learning in games.
\newblock \emph{Advances in Neural Information Processing Systems}, 28, 2015.

\bibitem[Tammelin et~al.(2015)Tammelin, Burch, Johanson, and Bowling]{tammelin2015solving}
Oskari Tammelin, Neil Burch, Michael Johanson, and Michael Bowling.
\newblock Solving heads-up limit texas hold'em.
\newblock In \emph{Twenty-fourth international joint conference on artificial intelligence}, 2015.

\bibitem[Von~Stengel(1996)]{von1996efficient}
Bernhard Von~Stengel.
\newblock Efficient computation of behavior strategies.
\newblock \emph{Games and Economic Behavior}, 14\penalty0 (2):\penalty0 220--246, 1996.

\bibitem[Zhang et~al.(2024)Zhang, Anagnostides, Farina, and Sandholm]{zhang2024efficient}
Brian~Hu Zhang, Ioannis Anagnostides, Gabriele Farina, and Tuomas Sandholm.
\newblock Efficient {$\Phi$}-regret minimization with low-degree swap deviations in extensive-form games.
\newblock \emph{arXiv preprint arXiv:2402.09670}, 2024.

\bibitem[Zinkevich et~al.(2007)Zinkevich, Johanson, Bowling, and Piccione]{zinkevich2007regret}
Martin Zinkevich, Michael Johanson, Michael Bowling, and Carmelo Piccione.
\newblock Regret minimization in games with incomplete information.
\newblock \emph{Advances in neural information processing systems}, 20, 2007.

\end{thebibliography}
